\documentclass[twoside]{article}

\usepackage[accepted]{aistats2022}
\usepackage[round]{natbib}

\usepackage[utf8]{inputenc}
\usepackage{amsmath}
\usepackage{amsfonts}
\usepackage{amssymb}
\usepackage{graphicx}
\usepackage{mymathletters}
\usepackage{amsthm}
\newtheorem{lemma}{Lemma}
\newtheorem{proposition}{Proposition} 
\newtheorem{theorem}{Theorem}
\newtheorem{corollary}{Corollary}
\newtheorem{definition}{Definition}

\usepackage{color}
\usepackage{listings}
\usepackage{caption}

\usepackage{paralist}

\usepackage{tikz}
\usetikzlibrary{graphs}
\usetikzlibrary{arrows, arrows.meta}
\tikzset{>={[round,sep]Stealth}}
\tikzset{bidi/.style={<->,dashed}}

\usetikzlibrary{decorations}
\pgfdeclaredecoration{simple line}{initial}{
  \state{initial}[width=\pgfdecoratedpathlength-1sp]{\pgfmoveto{\pgfpointorigin}}
  \state{final}{\pgflineto{\pgfpointorigin}}
}
\tikzset{
   shift left/.style={decorate,decoration={simple line,raise=#1}},
   shift right/.style={decorate,decoration={simple line,raise=-1*#1}},
}

\expandafter\def\csname decrement1  \endcsname{0}
\expandafter\def\csname decrement2  \endcsname{1}
\expandafter\def\csname decrement3  \endcsname{2}
\expandafter\def\csname decrement4  \endcsname{3}
\expandafter\def\csname decrement5  \endcsname{4}
\expandafter\def\csname decrement6  \endcsname{5}
\expandafter\def\csname decrement7  \endcsname{6}
\expandafter\def\csname decrement8  \endcsname{7}
\expandafter\def\csname decrement9  \endcsname{8}
\expandafter\def\csname decrement10 \endcsname{9}
\expandafter\def\csname decrement11 \endcsname{10}
\expandafter\def\csname decrement12 \endcsname{11}
\expandafter\def\csname decrement13 \endcsname{12}

\def\decrement#1{\csname decrement#1 \endcsname}

\def\placeFiveNodes{%
\expandafter\def\csname pos0x\endcsname{-0.00cm}\expandafter\def\csname pos0y\endcsname{2.00cm}%
\expandafter\def\csname curve1x\endcsname{-1.23cm}\expandafter\def\csname curve1y\endcsname{1.70cm}%
\expandafter\def\csname pos1x\endcsname{-1.90cm}\expandafter\def\csname pos1y\endcsname{0.62cm}%
\expandafter\def\csname curve2x\endcsname{-2.00cm}\expandafter\def\csname curve2y\endcsname{-0.65cm}%
\expandafter\def\csname pos2x\endcsname{-1.18cm}\expandafter\def\csname pos2y\endcsname{-1.62cm}%
\expandafter\def\csname curve3x\endcsname{-0.00cm}\expandafter\def\csname curve3y\endcsname{-2.10cm}%
\expandafter\def\csname pos3x\endcsname{1.18cm}\expandafter\def\csname pos3y\endcsname{-1.62cm}%
\expandafter\def\csname curve4x\endcsname{2.00cm}\expandafter\def\csname curve4y\endcsname{-0.65cm}%
\expandafter\def\csname pos4x\endcsname{1.90cm}\expandafter\def\csname pos4y\endcsname{0.62cm}%
\expandafter\def\csname curve5x\endcsname{1.23cm}\expandafter\def\csname curve5y\endcsname{1.70cm}%
\def\thenodes{\circlednode{0}\circlednode{1}\circlednode{2}\circlednode{3}\circlednode{4}}%
}

\def\placeFiveTreeNodes{%
\expandafter\def\csname pos0x\endcsname{-0.00cm}\expandafter\def\csname pos0y\endcsname{2.00cm}%
\expandafter\def\csname curve1x\endcsname{-1.23cm}\expandafter\def\csname curve1y\endcsname{1.70cm}%
\expandafter\def\csname pos1x\endcsname{-1.90cm}\expandafter\def\csname pos1y\endcsname{0.62cm}%
\expandafter\def\csname curve3x\endcsname{-2.00cm}\expandafter\def\csname curve3y\endcsname{-0.65cm}%
\expandafter\def\csname pos3x\endcsname{-1.18cm}\expandafter\def\csname pos3y\endcsname{-1.62cm}%
\expandafter\def\csname pos4x\endcsname{1.18cm}\expandafter\def\csname pos4y\endcsname{-1.62cm}%
\expandafter\def\csname curve4x\endcsname{2.00cm}\expandafter\def\csname curve4y\endcsname{-0.65cm}%
\expandafter\def\csname pos2x\endcsname{1.90cm}\expandafter\def\csname pos2y\endcsname{0.62cm}%
\expandafter\def\csname curve2x\endcsname{1.23cm}\expandafter\def\csname curve2y\endcsname{1.70cm}%
\def\thenodes{\circlednode{0}\circlednode{1}\circlednode{2}\circlednode{3}\circlednode{4}}%
}

\def\placeSixNodes{%
\expandafter\def\csname pos0x\endcsname{-0.00cm}\expandafter\def\csname pos0y\endcsname{2.00cm}%
\expandafter\def\csname curve1x\endcsname{-1.05cm}\expandafter\def\csname curve1y\endcsname{1.82cm}%
\expandafter\def\csname pos1x\endcsname{-1.73cm}\expandafter\def\csname pos1y\endcsname{1.00cm}%
\expandafter\def\csname curve2x\endcsname{-2.10cm}\expandafter\def\csname curve2y\endcsname{0.00cm}%
\expandafter\def\csname pos2x\endcsname{-1.73cm}\expandafter\def\csname pos2y\endcsname{-1.00cm}%
\expandafter\def\csname curve3x\endcsname{-1.05cm}\expandafter\def\csname curve3y\endcsname{-1.82cm}%
\expandafter\def\csname pos3x\endcsname{-0.00cm}\expandafter\def\csname pos3y\endcsname{-2.00cm}%
\expandafter\def\csname curve4x\endcsname{1.05cm}\expandafter\def\csname curve4y\endcsname{-1.82cm}%
\expandafter\def\csname pos4x\endcsname{1.73cm}\expandafter\def\csname pos4y\endcsname{-1.00cm}%
\expandafter\def\csname curve5x\endcsname{2.10cm}\expandafter\def\csname curve5y\endcsname{-0.00cm}%
\expandafter\def\csname pos5x\endcsname{1.73cm}\expandafter\def\csname pos5y\endcsname{1.00cm}%
\expandafter\def\csname curve6x\endcsname{1.05cm}\expandafter\def\csname curve6y\endcsname{1.82cm}%
\def\thenodes{\circlednode{0}\circlednode{1}\circlednode{2}\circlednode{3}\circlednode{4}\circlednode{5}}%
}

\def\placeSevenNodes{%
\expandafter\def\csname pos0x\endcsname{-0.00cm}\expandafter\def\csname pos0y\endcsname{2.00cm}%
\expandafter\def\csname curve1x\endcsname{-0.91cm}\expandafter\def\csname curve1y\endcsname{1.89cm}%
\expandafter\def\csname pos1x\endcsname{-1.56cm}\expandafter\def\csname pos1y\endcsname{1.25cm}%
\expandafter\def\csname curve2x\endcsname{-2.05cm}\expandafter\def\csname curve2y\endcsname{0.47cm}%
\expandafter\def\csname pos2x\endcsname{-1.95cm}\expandafter\def\csname pos2y\endcsname{-0.45cm}%
\expandafter\def\csname curve3x\endcsname{-1.64cm}\expandafter\def\csname curve3y\endcsname{-1.31cm}%
\expandafter\def\csname pos3x\endcsname{-0.87cm}\expandafter\def\csname pos3y\endcsname{-1.80cm}%
\expandafter\def\csname curve4x\endcsname{-0.00cm}\expandafter\def\csname curve4y\endcsname{-2.10cm}%
\expandafter\def\csname pos4x\endcsname{0.87cm}\expandafter\def\csname pos4y\endcsname{-1.80cm}%
\expandafter\def\csname curve5x\endcsname{1.64cm}\expandafter\def\csname curve5y\endcsname{-1.31cm}%
\expandafter\def\csname pos5x\endcsname{1.95cm}\expandafter\def\csname pos5y\endcsname{-0.45cm}%
\expandafter\def\csname curve6x\endcsname{2.05cm}\expandafter\def\csname curve6y\endcsname{0.47cm}%
\expandafter\def\csname pos6x\endcsname{1.56cm}\expandafter\def\csname pos6y\endcsname{1.25cm}%
\expandafter\def\csname curve7x\endcsname{0.91cm}\expandafter\def\csname curve7y\endcsname{1.89cm}%
\def\thenodes{\circlednode{0}\circlednode{1}\circlednode{2}\circlednode{3}\circlednode{4}\circlednode{5}\circlednode{6}}%
}

\def\placeEightNodes{
\expandafter\def\csname pos0x\endcsname{-0.00cm}\expandafter\def\csname pos0y\endcsname{2.00cm}%
\expandafter\def\csname curve1x\endcsname{-0.80cm}\expandafter\def\csname curve1y\endcsname{1.94cm}%
\expandafter\def\csname pos1x\endcsname{-1.41cm}\expandafter\def\csname pos1y\endcsname{1.41cm}%
\expandafter\def\csname curve2x\endcsname{-1.94cm}\expandafter\def\csname curve2y\endcsname{0.80cm}%
\expandafter\def\csname pos2x\endcsname{-2.00cm}\expandafter\def\csname pos2y\endcsname{0.00cm}%
\expandafter\def\csname curve3x\endcsname{-1.94cm}\expandafter\def\csname curve3y\endcsname{-0.80cm}%
\expandafter\def\csname pos3x\endcsname{-1.41cm}\expandafter\def\csname pos3y\endcsname{-1.41cm}%
\expandafter\def\csname curve4x\endcsname{-0.80cm}\expandafter\def\csname curve4y\endcsname{-1.94cm}%
\expandafter\def\csname pos4x\endcsname{-0.00cm}\expandafter\def\csname pos4y\endcsname{-2.00cm}%
\expandafter\def\csname curve5x\endcsname{0.80cm}\expandafter\def\csname curve5y\endcsname{-1.94cm}%
\expandafter\def\csname pos5x\endcsname{1.41cm}\expandafter\def\csname pos5y\endcsname{-1.41cm}%
\expandafter\def\csname curve6x\endcsname{1.94cm}\expandafter\def\csname curve6y\endcsname{-0.80cm}%
\expandafter\def\csname pos6x\endcsname{2.00cm}\expandafter\def\csname pos6y\endcsname{-0.00cm}%
\expandafter\def\csname curve7x\endcsname{1.94cm}\expandafter\def\csname curve7y\endcsname{0.80cm}%
\expandafter\def\csname pos7x\endcsname{1.41cm}\expandafter\def\csname pos7y\endcsname{1.41cm}%
\expandafter\def\csname curve8x\endcsname{0.80cm}\expandafter\def\csname curve8y\endcsname{1.94cm}%
\def\thenodes{\circlednode{0}\circlednode{1}\circlednode{2}\circlednode{3}\circlednode{4}\circlednode{5}\circlednode{6}\circlednode{7}}%
}

\placeFiveNodes

\def\edge#1#2{\draw[dashed,<->] (n#1) edge[bend right=35] (n#2);}

\def\markctreeedge[#1]#2#3{\draw[#1,->,thick] (n#2) -- (n#3);}
\def\markcedge[#1]#2{\markctreeedge[#1]{\decrement{#2}}{#2} ;} %

\def\markI#1{\markcedge[green!75!black]{#1}}

\def\markZ#1{\markI{#1}}
\def\markFail#1{\markcedge[red!75!black]{#1}}

\def\pos#1{(\csname pos#1x\endcsname,\csname pos#1y\endcsname)}

\makeatletter
\def\circlednode#1{\node[circle,fill=white,inner sep=1mm] (n#1) at \pos#1 {#1};}

\def\makegraph#1{
\begin{tikzpicture}
\draw (-3cm,-3cm) rectangle (3cm, 3cm);\thenodes #1
\end{tikzpicture}
}

\newcounter{nalg}[section] 
\renewcommand{\thenalg}{\arabic{nalg}} 
\DeclareCaptionLabelFormat{algocaption}{Algorithm \thenalg} 

\lstnewenvironment{algorithm}[1][] 
{   
    \refstepcounter{nalg} 
    \captionsetup{labelformat=algocaption,labelsep=colon} 
    \lstset{ 
        mathescape=true,
        frame=tB,
        numbers=left, 
        numberstyle=\tiny,
        basicstyle=\scriptsize, 
        keywordstyle=\color{black}\bfseries\em,
        keywords={,input, output, return, datatype, function, in, if, else, foreach, while, begin, end, } 
        numbers=left,
        xleftmargin=.04\textwidth,
        #1 
    }
}
{}

\def\evalSigma#1{\left[#1\right]_\cG} 
\def\evalSigmaG#1#2{\left[#1\right]_{#2}}

\newcommand{\IDSCORE}[2][{}]{|\textit{ID}_{#1}({#2})|}

\newcommand{\bidirected}{\leftrightarrow}

\def\An{\textit{An}}

\def\Pa{\textit{Pa}}

\usepackage{newunicodechar}
\def\sigmaswallow#1#2{\sigma_{\decrement{#1}\decrement{#2}}}
\newunicodechar{σ}{\sigmaswallow}

\begin{document}

\twocolumn[

\aistatstitle{Identification in Tree-shaped Linear Structural Causal Models }

\aistatsauthor{Benito van der Zander$^+$ \And Marcel Wien\"{o}bst$^+$ \And  Markus Bl\"{a}ser$^*$ \And Maciej Li\'{s}kiewicz$^+$ }

\runningauthor{Benito van der Zander, Marcel Wien\"{o}bst,  Markus Bl\"{a}ser, Maciej Li\'{s}kiewicz}


\aistatsaddress{$^+$ University of L\"{u}beck, Germany \
$^*$ Saarland University, Saarland Informatics Campus, Saarbr\"{u}cken, Germany } ]

\begin{abstract}
Linear structural equation models represent direct causal effects as directed edges and confounding factors as bidirected edges. An open problem is to identify the causal parameters from correlations between the nodes. We investigate models, whose directed component forms a tree, and show that there, besides classical instrumental variables, missing cycles of bidirected edges can be used to identify the model. They can yield systems of quadratic equations that we explicitly solve to obtain one or two solutions for the causal parameters of adjacent directed edges. We show how multiple missing cycles can be combined to obtain a unique solution. This results in an algorithm that can identify instances that previously required approaches based on Gröbner bases, which have doubly-exponential time complexity in the number of structural parameters.
\end{abstract}

\section{INTRODUCTION}

Linear structural causal models (SCMs or structural equation models, SEMs) 
are frequently used to express
and analyze the relationships between random variables of interest \citep{bollen2014structural,duncan2014introduction}.
Each variable $V_i$, with $i=0,\ldots, n$ is assumed to be linearly dependent 
on  the remaining variables and an error term $\varepsilon_i$ of normal distribution 
with zero mean and some covariance matrix $\Omega=(\omega_{ij})$ between the terms:
$$V_i=\sum_{j} \lambda_{ji} V_j +\varepsilon_i.$$
In this paper we consider  \emph{recursive models}, i.e.\ we assume that,
for all $j>i$, we have $\lambda_{ji} =0$. Such a model can be represented 
 as a graph with nodes over the variables. 
 Directed edges represent a linear influence $\lambda_{ji}$ of a parent node $j$ on its child
 $i$. Bidirected edges represent an additional correlation $\omega_{ij}\neq 0$ between random error terms. 
%
Given the graph and the weights (also called coefficients or direct causal effects)  $\lambda_{ij}$, one 
can calculate the covariances between variables along the paths. 
For example, in $\cG_1$ in Fig.~\ref{fig:examples}, 
which models random variables $V_0,V_1,$ and $V_2$, 
a unit change of $V_0$ implies a change of $\lambda_{01}$ of $V_1$ 
and a change of $\lambda_{01}\lambda_{12}$ of $V_2$. The covariance 
$\sigma_{01}$ ($\sigma_{02}$) between $V_0$ and $V_1$ ($V_2$) 
is thus $\lambda_{01}$ ($\lambda_{01}\lambda_{12}$).

Writing the coefficients of all directed edges as an adjacency matrix 
$\Lambda=(\lambda_{ij})$ and the coefficients of all bidirected edges as an adjacency 
matrix $\Omega=(\omega_{ij})$, the covariances $\sigma_{ij}$ between each pair 
of random variables $V_i$ and $V_j$ can be calculated as matrix $\Sigma = (\sigma_{ij})$:

\begin{equation}
\Sigma = (I - \Lambda)^{-1} \Omega (I-\Lambda)^{-T} \label{eqn:SEM}
\end{equation}

\begin{figure} 
  \begin{tikzpicture}[xscale=1.2]
    \node (l) at (0.2,0.8) {$\mathcal{G}_1$};
    \node (z) at (0,0) {0};
    \node (x) at (1,0) {1};
    \node (y) at (2,0) {2};
    \draw [->] (z) -- node [midway,below] {$\lambda_{01}$} (x);
    \draw [->] (x) -- node [midway,below] {$\lambda_{12}$} (y);
    \draw [] (x) edge [bidi,bend left=65] node [midway,above]{$\omega$}(y);

    \node (l) at (5.5+0.7,0.8) {$\mathcal{G}_2$};
    \node (0) at (3.5+0.5,0) {0};
    \node (1) at (3.5+0.5,1.25) {1};
    \node (2) at (4.75+0.5,1) {2};
    \node (3) at (4.5+0.5,0) {3};
    \node (4) at (5.5+0.5,0) {4};
    \draw [->] (0) -- (1);
    \draw [->] (0) -- (2);
    \draw [<->] (0) edge [bidi,bend left=45] (1);
    \draw [<->] (0) edge [bidi,bend left=45] (2);
    \draw [->] (0) -- (3);
    \draw [->] (3) -- (4);
    \draw [<->] (0) edge [bidi,bend right=65] (3);
    \draw [<->] (0) edge [bidi,bend right=65] (4);

    \node (l) at (5+0.7,-2-0.25) {$\mathcal{G}_3$};
    \node (0) at (0+0.9,-2-0.25) {0};
    \node (1) at (1+0.9,-2-0.25) {1};
    \node (2) at (2+0.9,-2-0.25) {2};
    \node (3) at (3+0.9,-2-0.25) {3};
    \node (4) at (4+0.9,-2-0.25) {4};
    \draw [->] (0) -- (1);
    \draw [->] (1) -- (2);
    \draw [->] (2) -- (3);
    \draw [->] (3) -- (4);
    \draw [<->] (0) edge [bidi,bend right=65] (1);
    \draw [<->] (0) edge [bidi,bend right=65] (2);
    \draw [<->] (0) edge [bidi,bend right=65] (3);
    \draw [<->] (0) edge [bidi,bend right=65] (4);
    \draw [<->] (2) edge [bidi,bend left=65] (4);
  \end{tikzpicture}
  \caption[Example of instrumental variables]{
    $\cG_1$: the classic instrumental  variables (IV) model. 
    $\cG_2$ is (generically) identifiable by the TSID 
    algorithm \citep{identifyingEdgeWiseDeterminantalDrton}
    and our method TreeID
    but for which both the half-trek criterion (HTC) and the ACID 
    algorithm  \citep{kumor2020auxiliaryCutsets} fail.
    Graph $\cG_3$ is identifiable by  TreeID but not by TSID.
  }
  \label{fig:examples}
\end{figure}
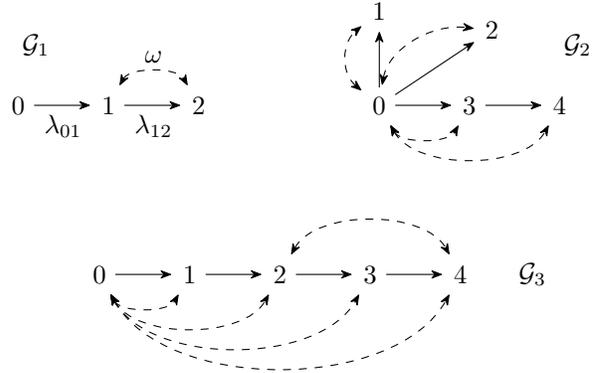

Of interest is the reverse problem, the identification and estimation
of causal effects. That is, given the graph and a matrix~$\Sigma$,
compute the matrix~$\Lambda$. The \emph{identification problem},
which is the main focus of this paper, asks for a symbolic equation to calculate 
the coefficients in $\Lambda$, the \emph{estimation problem} asks for a numerical solution. 
Not every coefficient can be identified, yielding the problem to determine which coefficients have solutions.

Identification in linear SCMs has been the subject of a considerable amount of research
in the last decades, including the early work in econometrics 
\citep{wright1928tariff,fisher1966identification,bowden1984instrumental}
and the pioneering work on the computational aspects of the problem~\citep{Pearl2009}.
Most approaches for solving the problem are based on instrumental variables, 
in which the causal effect is identified as a fraction of two covariances \citep{wright1928tariff,bowden1984instrumental}. For example, in $\cG_1$ 
in Fig.~\ref{fig:examples},
one can calculate $\lambda_{12} = \frac{\lambda_{01}\lambda_{12}}{\lambda_{01}} = \frac{\sigma_{02}}{\sigma_{01}}$.  The variable $V_0$ is then called an instrumental variable (IV). 
The literature focuses on developing criteria to decide whether a variable is an IV.
A more complex criterion -- the criterion for a \emph{conditional instrumental variable} (cIV) -- considers these correlations of~$V_0$ conditionally on another set of variables \citep{bowden1984instrumental,PearlConditionalIV,IJCAIInstruments}.
Other criteria and methods to identify some coefficients in specific graphs
involve instrumental sets (IS) \citep{BritoPearlUAI02,brito2010instrumental,brito2002graphical,instrumentalVariablesZander2016}, 
half-treks (HTC) \citep{halftrek2012}, 
auxiliary instrumental variables (aIV) \citep{chen2015incorporating}, 
determinantal instrumental variables (tsIV) which results in the TSID algorithm \citep{identifyingEdgeWiseDeterminantalDrton}, 
instrumental cutsets \citep{kumor2019instrumentalCutsets},  
or auxiliary cutsets which result in the ACID algorithm~\citep{kumor2020auxiliaryCutsets}.
Some criteria lead to polynomial-time algorithms (ACID).
For other criteria, e.g.\ cIV and tsIV,  the decision problem, 
if the criterion is satisfied in a given graph,
is NP-complete \citep{IJCAIInstruments,instrumentalVariablesZander2016,kumor2019instrumentalCutsets}. 

A drawback to all criteria listed above is that they are not complete, i.e.\ not applicable to every graph. They also only decide whether there exists exactly one solution for a coefficient.
 
An alternative approach is to expand Eq.~\eqref{eqn:SEM} to a system
of polynomial equations and solve it using a computer algebra system
(CAS), which usually employs Gröbner bases
\citep{garcia2010identifying,halftrek2012}. This gives a complete
solution for any solvable equation system. However, Gröbner base algorithms have a doubly exponential runtime and are EXPSPACE-complete \citep{MAYR1982305}. Thus, they are often too slow to be used in practice.  \citet{garcia2010identifying} note the runtime varies between seconds and 75 days for graphs with four nodes.
 


\paragraph{Our results.} 
We investigate the identification 
problem in linear SCMs on graphs whose directed component is a tree and whose bidirected component can be arbitrary. 
Figure~\ref{fig:examples} shows example causal structures, which are modeled as tree graphs.

First, we show that if a node $i$ is not connected to the root node with a bidirected edge, 
the root node can be used as a classic instrumental variable to identify the  
incoming edge to $i$.  
Then we investigate the subgraph of nodes that are each connected to the root node with a bidirected edge. 
We show that, if there is a missing cycle, such that none of the bidirected edges 
${i_1} \bidirected {i_2} \bidirected {i_3} \bidirected \ldots \bidirected {i_1}$ 
exists in the graph, it yields an equation system that can lead 
to a solution of all incoming directed edges to the involved nodes.
We describe how to reduce this equation system to a single quadratic 
equation in a single variable, which can then identify one of the incoming edges. 
As a quadratic equation, it can yield exactly one, exactly two, or infinitely many solutions. 
How many solutions exist for a certain graph can be symbolically determined using Polynomial Identity Testing (PIT).
For example, a graph with directed component $0\to 1 \to 2 \to 3$ and bidirected edges 
$0 \bidirected i$ for $i=1,2,3$  does not contain the bidirected 
cycle $1 \bidirected 2 \bidirected 3 \bidirected 1$, 
so our approach returns exactly two solutions for the identification of 
$\lambda_{01}, \lambda_{12}, \lambda_{23}$.

This results in an algorithm, we call TreeID,  that can identify instances 
that previously required the Gröbner bases approach 
and for which the state-of-the-art methods, such as HTC, TSID, and ACID fail.

We have performed the necessary calculations on a large number of graphs, with a special focus on graphs whose directed edges form a single directed path and where the root node is connected to all other nodes with bidirected edges. 
If the bidirected component is complete except for a missing cycle of a length between five and ten edges, there are always exactly two solutions.

\paragraph{Identification of Tree Graphs with the State-of-the-Art-Methods.}
The algorithm ACID \citep{kumor2020auxiliaryCutsets}, which subsumes previous state-of-the-art methods,
including cIV, IS, aIV, and HTC, as well as the TSID approach \citep{identifyingEdgeWiseDeterminantalDrton}
belong currently to the most prominent and powerful methods  
for identification of structural coefficients in linear causal models.
They are significantly more efficient compared to the general Gröbner bases approach. 

For tree graphs, however, more sophisticated criteria are often either
not applicable or not advantageous compared to simpler ones. This can be most easily seen in the
case of ISs. As all nodes have in-degree~1, it is not
possible to identify \emph{two or more} incoming edges at once.

We show that many of the previous criteria
and algorithms, including ACID,  
collapse to the use of auxiliary instrumental variables (aIVs).
This result shows that making progress beyond simple rules such as aIV on tree graphs appears
to be quite challenging. One explanation could of course be that,
e.\,g., the ACID algorithm is already powerful enough to identify such simple models. However,
this turns out not to be the case. Indeed, there is a large number of
tree graphs not identifiable in this manner as is shown exemplarily in Fig.~\ref{fig:examples}:
Graph $\cG_2$, which is taken from Fig.~3c in \cite{halftrek2012},
is generically identifiable by TSID and our algorithm TreeID
but for which both the HTC and ACID fail to identify coefficients.

Several of the tree graphs, like $\cG_1$ and $\cG_2$ in 
Fig.~\ref{fig:examples},  can be identified with the TSID algorithm implementing the tsIV
criterion \citep{identifyingEdgeWiseDeterminantalDrton}. This
criterion, however, still leaves a significant number of tree graphs unidentified. 
For an example, see $\cG_3$ in Fig.~\ref{fig:examples}, which is identifiable by 
TreeID.  
Moreover, the tsIV criterion has the drawback that it is likely not efficiently
testable as it was shown to be NP-hard \citep{kumor2020auxiliaryCutsets}. 

Our algorithm TreeID is an entirely new approach for identification independent of the
instrumental variable framework. 

\section{PRELIMINARIES}
\label{sec:prelimin}

We consider mixed acyclic graphs  $\cG=(V,D,B)$ with  $n+1$ nodes $V=\{0,1,\ldots, n\}$, 
directed edges $i \to j$ in $D$ and bidirected edges $i \bidirected j$ in $B$. 
By acyclicity, we mean that $\cG=(V,D)$ restricted to directed edges is a directed acyclic graph (DAG).

For a given graph $\cG$,
the \emph{identification problem} asks to find, for each parameter $\lambda_{xy}$, 
with $x\to y \in D$, an expression to calculate $\lambda_{xy}$ that only 
depends on the entries of the matrix $\Sigma$, such that the calculated 
$\lambda_{xy}$ uniquely satisfies Eq.~\eqref{eqn:SEM}. Practically, 
there are initial values of $\Lambda$ and $\Omega$, from which 
a matrix $\Sigma$ is calculated. Using the expression for 
$\lambda_{xy}$, new matrices $\Lambda'$ and $\Omega'$ can be calculated and it needs to hold 
$$(I - \Lambda')^{-1} \Omega' (I-\Lambda')^{-T} = \Sigma = (I - \Lambda)^{-1} \Omega (I-\Lambda)^{-T}.$$

We consider the generic version of the problem, which only requires that the expressions are valid almost everywhere, i.e.\ the Lebesgue measure of the set of initial parameters for which the expressions are not valid is zero. E.g.\ it is allowed to return solutions like $\lambda = \frac{\sigma_{zy}}{\sigma_{zx}}$ even though they are not valid for $\sigma_{zx} = 0$. Since any algebraic subset has Lebesgue measure zero, we can assume that any polynomial over elements of $\Lambda, \Omega$ evaluates to zero if and only if it is the zero polynomial.
By $\IDSCORE[\cG]{\lambda_{xy}}$, we denote the degree of identifiability of edge $\lambda_{xy}$. 
$\IDSCORE[\cG]{\lambda_{xy}} = 1$ if $\lambda_{xy}$ is \emph{uniquely identifiable},
$\IDSCORE[\cG]{\lambda_{xy}} = \infty$ if \emph{not identifiable}, and 
$1 < \IDSCORE[\cG]{\lambda_{xy}} < \infty$ if identifiable with more than one solution, e.g.\ $\IDSCORE[\cG]{\lambda_{xy}} = 2$ means \emph{2-identifiable}.

A \emph{path} in a graph $\cG$ is a node sequence $i_1, \ldots ,i_{\ell+1}$
such that all successive nodes $i_k, i_{k+1}$, with $1\le k\le \ell$, are connected by a directed edge. 
Then $i_1$ is called the \emph{start node} and  $i_{\ell +1}$ the \emph{end node} of the
path. We use the terms \emph{child}, \emph{parent}, \emph{ancestor}, \emph{descendant}, and \emph{sibling}
to describe node relationships in graphs in
the same way as~\cite{Pearl2009}; 
in this convention, every node is an ancestor (but not a parent) and a descendant
(but not a child) of itself. For a node $i$,
we denote by $\An(i)$ the set of all ancestors of $i$.

A \emph{trek} $\tau$ in $\cG$ from source $i$ to target $j$ is a path from $i$ to $j$ whose consecutive
edges do not have any colliding arrowheads, i.e.\ $\tau$ a path of one of the two following forms
$ i \gets i_1 \gets \ldots \gets u \bidirected v \to j_1 \to \ldots \to j$
where node $i$ can coincide with $u$ or  $j$ can coincide with $v$, or  
$ i \gets i_1 \gets \ldots \gets u  \to j_1 \to \ldots \to j$
where either $i$ or $j$ can coincide with $u$.
Define the \emph{trek monomials} $M(\tau)$ as follows. For $\tau$ of the first form, define  
$   
  M(\tau) = \omega_{uv}\prod_{x \to y \in \tau} \lambda_{xy}
$ and, for $\tau$ of the second form, define  
$   
  M(\tau) = \omega_{uu}\prod_{x \to y \in \tau} \lambda_{xy}.
$
 Then, the following \emph{trek rule} \citep{wright1921correlation,wright1934method}
expresses the covariance matrix~\eqref{eqn:SEM}
as a summation over all treks 
\begin{equation}\label{eqn:trek:rule}
 \sigma_{ij} = \sum_{\tau\ \text{trek from $i$ to $j$}} M(\tau).
\end{equation}

In this paper, we restrict ourselves to \emph{tree graph models}, i.e.\ assume that $\cG=(V,D)$ is a directed tree, which has exactly one node, called \emph{root},  whose incoming degree is zero 
and all other nodes have incoming degree one. The root node is labeled $0$. 
For each node $i$, with $i > 0$ and its (unique) parent $p$, the coefficient of 
the incoming edge $p\to i$ is denoted as $\lambda_{pi}$,  also written sometimes 
as $\lambda_{i}$ for short. 


Finally, we recall some concepts generalizing IVs, which are relevant to our paper.
The idea of auxiliary variables aIV~\citep{chen2017identification} is to utilize
already identified direct effects for further
identification. Assume for variable $y$, the incoming edge $x
\rightarrow y$ has been identified. Then, one can create the variable
$y^* = y - \lambda_{xy} x$, i.e.\ subtract out the identified
direct effect. The resulting auxiliary variable $y^*$ acts as if there
is no edge $x \rightarrow y^*$, which enables further
identification. The aIV criterion identifies edges by the instrumental
variable criterion and creates corresponding auxiliary variables until
no further edge can be identified.

\section{POLYNOMIAL IDENTITY TESTING}\label{sec:polynomial:testing}
To find the solutions for $\lambda_{ij}$ in a given tree graph $\cG=(V,D,B)$,
our algorithm handles multivariate 
functions involving polynomials over 
$\lambda_{ij}$, with $i\to j \in D$,  and $\sigma_{ij}$, with $0\le i,j \le n$.
An important task that the algorithm has to cope with 
during the computation is to verify whether a formula $F$ 
for a parameter $\lambda$ satisfies an equation involving 
expressions over $\sigma_{ij}$. Another problem is to check if 
a given $F$ is a zero-function.
%
E.g., is $\lambda_{12} \sigma_{10} - \sigma_{20}$
the zero-polynomial? One can easily see that 
for the graph $\cG_1$ in Fig.~\ref{fig:ex:pit} this is the case
but for $\cG_2$ not. 

Below we show that these tasks can be 
reduced to Polynomial Identity Testing (PIT) and in consequence
solved efficiently using the well-known approach based on the Schwartz-Zippel lemma\footnote{\cite{pipSchwartz1980Zippel, pipZippel1979Schwartz, pipDemillo1978SchwartzZippel}}.
The lemma states the probability of a non-zero polynomial evaluating to zero at random variable values is negligible.
\begin{lemma}[Schwartz-Zippel]\label{lemm:Schwartz:Zippel}
Let $p(x_1,\ldots,x_n)$ be a non-zero polynomial of total degree $\le d$ over 
a field~$\mathbb{F}$. Let $S\subseteq \mathbb{F}$ be a finite set
and let $a_1,\ldots,a_n$ be selected at random independently and uniformly
from $S$. Then $\Pr[p(a_1,\ldots,a_m)\not= 0]\ge 1-d/|S|$.
\end{lemma}

Our algorithm represents formulas in a form, we 
call \emph{fractional affine square-root terms of polynomials} (FASTP);
We define it as $\frac{p+q\sqrt{s}}{r+t\sqrt{s}}$, where $p,q,r,s,t$ 
are multivariate polynomials. 
In particular,
the algorithm  represents parameters  $\lambda$ as  FASTPs 
with  $p,q,r,s,t$  over  $\sigma_{ij}$.

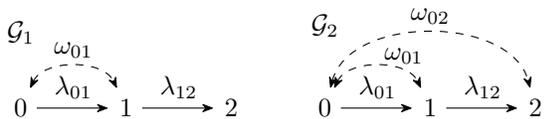
\begin{figure}  
\begin{center}
  \begin{tikzpicture} 
    \node (G) at (0,1) {$\cG_1$};
    \node (z) at (0,0) {$0$};
    \node (x) at (1.4,0) {$1$};
    \node (y) at (2.8,0) {$2$};
    \draw [->] (z) -- node [midway,above] {$\lambda_{01}$} (x);
    \draw [->] (x) -- node [midway,above] {$\lambda_{12}$} (y);
    \draw [<->] (z) edge [bidi,bend right=295] node [midway,above,sloped] {$\omega_{01}$}(x);
  \end{tikzpicture}
  \hspace*{5mm}
  \begin{tikzpicture} 
     \node (G) at (0,1.1) {$\cG_2$};
    \node (z) at (0,0) {$0$};
    \node (x) at (1.4,0) {$1$};
    \node (y) at (2.8,0) {$2$};
    \draw [->] (z) -- node [midway,above] {$\lambda_{01}$} (x);
    \draw [->] (x) -- node [midway,above] {$\lambda_{12}$} (y);
    \draw [<->] (z) edge [bidi,bend right=295] node [near end,above] {$\omega_{01}$}(x);
    \draw [<->] (z) edge [bidi,bend right=285] node [midway,above,sloped] {$\omega_{02}$}(y);
  \end{tikzpicture}\vspace*{-3mm}
\end{center}
\caption{\label{fig:ex:pit}
The polynomial $\lambda_{12} \sigma_{10} - \sigma_{20}$
vanishes in the model $\cG_1$ but it 
is a nonzero-polynomial in $\cG_2$.
} 
\end{figure} 

\begin{definition}\label{def:substit}
Let $\cG=(V,D,B)$ over $V=\{0,\ldots, n\}$ be an arbitrary mixed graph. 
For a FASTP 
$F$ over
$\lambda_{ij}$, with $i\to j \in D$,  and $\sigma_{ij}$, with $0\le i,j \le n$, 
let $\evalSigma{F}$ be the substitution of all $\sigma_{ij}$ 
with terms in $\lambda_{ij}, \omega_{ij}$ according to the 
trek rule~\eqref{eqn:trek:rule}. 
Thus, assuming we get no division by a zero-polynomial,
$\evalSigma{F}$ is a FASTP 
over $\lambda_{ij}$, with $i \to j \in D$, 
and $\omega_{ij}$, with $i\bidirected j \in B$.
\end{definition}

For example, for the polynomial $F=\lambda_{12} \sigma_{10} - \sigma_{20}$  above
and the models $\cG_1$ and $\cG_2$ in Fig.~\ref{fig:ex:pit}, we have 
 $ \evalSigmaG{F}{\cG_1} = \lambda_{12}(\omega_{01}+\lambda_{01} \omega_{00}) -
      \lambda_{12}\omega_{01} - \lambda_{12}\lambda_{01} \omega_{00} $
      and 
  $ \evalSigmaG{F}{\cG_2} = \lambda_{12}(\omega_{01}+\lambda_{01} \omega_{00}) -
      \lambda_{12}\omega_{01} - \lambda_{12}\lambda_{01} \omega_{00} - \omega_{02}$.
Thus, $ \evalSigmaG{F}{\cG_1}$ is the zero-polynomial but  $\evalSigmaG{F}{\cG_2}= - \omega_{02}$ not.
This implies that $F$ vanishes when considering model $\cG_1$ but it 
is a nonzero-polynomial in $\cG_2$.


Now, we are ready to show, that testing  if $\lambda$ represented as  a FASTP
satisfies a specific equation can be reduced to PIT.

\begin{lemma}\label{lemm:FASTP}
For a given FASTP $\lambda=\frac{p+q\sqrt{s}}{r+t\sqrt{s}} $ and polynomials $a,b,c$, 
one can verify -- up to a sign -- whether $\evalSigma{a\lambda^2+b\lambda+c }\equiv 0$ holds using PIT.
\end{lemma}

\begin{proof}
\catcode42=9

The equality 
$(a*(q*\sqrt{s}+p)^2)/(t*\sqrt{s}+r)^2+(b*(q*\sqrt{s}+p))/(t*\sqrt{s}+r)+c \equiv 0 $
can be expressed as:
$c*s*t^2+\sqrt{s}*(2*c*r*t+b*p*t+b*q*r+2*a*p*q)+b*q*s*t+a*q^2*s+c*r^2+b*p*r+a*p^2 \equiv 0$.

We distinguish two cases: If $s$ is a perfect square (in the ring of polynomials), that is, there is a polynomial $\varsigma$ such that $\varsigma^2 = s$, then the above equation is an instance of polynomial identity testing. If $s$ is not a perfect square, then the left-hand side of the equation above can only be identically zero if 
$2*c*r*t+b*p*t+b*q*r+2*a*p*q \equiv 0$ and $c*s*t^2+b*q*s*t+a*q^2*s+c*r^2+b*p*r+a*p^2 \equiv 0$
(otherwise, $\sqrt{s}$ would 
be an element of the rational function field). These are two instances of PIT.

For this approach, we need to be able to check whether $s$ is a perfect square and if so, compute an arithmetic
circuit for $\varsigma$.
\cite{DBLP:journals/jacm/BrentK78} propose to use Newton iteration to
approximate the square root of a polynomial by a power series. If the polynomial is a perfect square, then this
approximation is exact. This even works for multivariate polynomials given by arithmetic circuits, some of
the details are spelled out in \cite{BLASER2017128}. Given a circuit $C$ for $s$, we use this algorithm
to compute a circuit $D$ for a candidate square root $\hat \varsigma$. We can now use PIT to check whether
$\hat \varsigma^2 = s$. If yes, then $D$ is a circuit for $\sqrt s$. In the no case, $s$ is not a perfect square.
\end{proof}

If $s$ is a perfect square, then there are two square roots, namely $\varsigma$ and $-\varsigma$.
There is a priori no canonical way to distinguish these two\footnote{While implementing our algorithm, we have noticed that the most practical way of deciding this and the equivalence of Lemma~\ref{lemm:FASTP} is to choose random values for variables $\lambda_{ij}$ and $\omega_{ij}$, compute $\sigma_{ij}$ using Equation~(\ref{eqn:trek:rule}), and just evaluate the polynomials with arbitrary precision arithmetic.}. 
If $s$ is a perfect square, then the output of the algorithm is called $\sqrt{s}$ and the other root
is $- \sqrt{s}$.

In our algorithm, the degree of all involved polynomials is polynomially bounded,
which ensures that the resulting PITs can be solved effectively.

\section{BASIC EQUATIONS AND PRELIMINARY IDENTIFICATIONS}\label{sec:basic:eqs}

In this section, we examine the covariances $\sigma_{ij} $ in a tree model $\cG$.

Let, for two nodes $s$ and $t$ connected by a directed path $\pi$, the function $L(s,t)$
be defined as follows:
$
  L(s,t) = \prod_{x\to y \in \pi} \lambda_{xy},
$ with  $L(s,s) = 1$.


\begin{lemma}\label{lemm:exp:for:sigma}
For a given tree graph $\cG=(V,D,B)$ and for any two nodes $i,j$ in $V$ we have:
$\evalSigma{\sigma_{ij}} = 
\sum_{s\in \An(i)} \sum_{t\in\An(j)} \omega_{st} L(s,i) L(t, j)
$
\end{lemma}


Thus, for a tree graph $\cG$ and a multivariate polynomial 
$F$ over $\lambda_{ij}$ and $\sigma_{ij}$, 
we get $\evalSigma{F}$ by the substitution 
of all $\sigma_{ij}$ according to the 
equation in Lemma~\ref{lemm:exp:for:sigma}. 

The two lemmas below can be proved easily 
and are special cases of well-known results, 
see e.g.~\citep{drton2018algebraic}.

\begin{lemma}\label{lem:tree:reversed:equations:new}
Assume $\cG$ is a tree graph and let $i,j$ be two different nodes.
Then, for all  $i,j>0$ and for (unique) parents $p,q$ with edges $p \to i, q \to j$,
we have 
$$\omega_{ij} = \evalSigma{\lambda_{pi} \lambda_{qj} \sigma_{pq} - \lambda_{pi} \sigma_{pj} - \lambda_{qj} \sigma_{iq} + \sigma_{ij} }.$$
Moreover, for all  $j>0$ and the parent $q$ with edge $q \to j$, it is true
$$\omega_{0j} = \evalSigma{\sigma_{0j} - \lambda_{qj} \sigma_{0q}}. $$
\end{lemma}


Based on Lemma~\ref{lem:tree:reversed:equations:new},
we analyze first the structure of equations of the system 
$(I - \Lambda) \Sigma (I-\Lambda)^{T} = \Omega $.
\begin{lemma}\label{lem:reversed:eqs:tree} 
An edge $\lambda_{xy}$ is identifiable in a tree graph~$\mathcal{G}$, 
if (and only if) the system of the equations 
\begin{itemize}
\item $\lambda_{pi} \lambda_{qj} \sigma_{pq} - \lambda_{pi} \sigma_{pj} - \lambda_{qj} \sigma_{iq} + \sigma_{ij} = 0$, for all missing edges $i \bidirected j$, 
  where $p$, resp.~$q$, are parents of~$i$, resp.~$j$,  and  
\item $\sigma_{0i} - \lambda_{pi} \sigma_{0p} = 0$, for all missing edges $0 \bidirected i$, 
where $p$ is a parent of $i$,
\end{itemize}
has a (unique) solution for $\lambda_{xy}$. Moreover,
the number of generic solutions for $\lambda_{xy}$ is equal to $\IDSCORE[\cG]{\lambda_{xy}}$.
\end{lemma}

%
%
%

From Lemma~\ref{lem:reversed:eqs:tree}, two obvious ways of identifying certain edges emerge.
\begin{corollary}\label{cor:zero:node}
If the edge $0 \bidirected i$ is missing for $i>0$ with parent $p$, 
then $\lambda_{pi}$ is identified as $ \lambda_{pi} = \sigma_{0i} / \sigma_{0p} $.
\end{corollary}
%
\begin{corollary}\label{cor:propagate}
If the edge $i \bidirected j$ is missing for $i$ with parent $p$ and 
$j$ with parent $q$, $\lambda_{pi}$ is identified, and 
$\evalSigma{( \lambda_{pi} \sigma_{pq} - \sigma_{iq})} \not\equiv 0$, 
then $\lambda_{qj}$ is identified as
$\lambda_{qj}  = ( \lambda_{pi} \sigma_{pj}  - \sigma_{ij} ) / ( \lambda_{pi} \sigma_{pq} - \sigma_{iq})$.
\end{corollary}

Below we show, that in many tree graphs the inequality required in the 
corollary above is true.

\begin{lemma}[Propagation]\label{lem:cond:for:propagate}
Let $\cG=(V,D,B)$ be a tree graph, and let $i,j$, with 
$i \bidirected j \notin B$, be two different nodes with parents 
$p \to i$ and $q\to j$. Then 
$\evalSigma{( \lambda_{pi} \sigma_{pq} - \sigma_{iq})} \not\equiv 0$
if and only if 
there is a trek from $i$ and $q$ in $\cG \setminus \{p\to i\}$.
In particular, the polynomial is non-zero in $\cG$ if $\cG$ contains the bidirected edge 
$0 \bidirected i$ or if $q$ is a descendant of $i$ in $\cG$.
\end{lemma}
\begin{proof}
Assume first that $q$ is a descendant of $i$ in $\cG$.
Then both there is a trek from $i$ and $q$ in $\cG \setminus \{p\to i\}$
as well as  $\evalSigma{( \lambda_{pi} \sigma_{pq} - \sigma_{iq})} \not\equiv 0$
since, for a trek $\tau'$ from $i$ and $q$ in $\cG \setminus \{p\to i\}$, the summation 
$
  \sigma_{iq} = \sum_{\tau\ \text{trek from $i$ to $q$} } M(\tau)
$
includes a term $M(\tau')$ which does not involve $\lambda_{pi}$.

If $q$ is not a descendant of $i$ in $\cG$, then we can express $\sigma_{iq} $ as
$$
  \sigma_{iq} = \lambda_{pi}\sigma_{pq} + \sum_{\tau\ \text{trek from $i$ to $q$ without $p \to i$} } M(\tau).
$$
If there is a trek from  $i$ and $q$ in $\cG \setminus \{p\to i\}$,
then the sum on the right-hand side does not vanish and
$\evalSigma{( \lambda_{pi} \sigma_{pq} - \sigma_{iq})} \not\equiv 0$.
On the other hand, if the polynomial is non-zero,  then 
$$ \sum_{\tau\ \text{trek from $i$ to $q$ without $p \to i$} } M(\tau) \not\equiv 0.$$
This means that in $\cG$ there exists a trek $\tau'$ from $i$ to $q$ without $p \to i$.
\end{proof}

Hence, we use the two Corollaries~\ref{cor:zero:node} and~\ref{cor:propagate} for a simple preliminary identification step by checking for which nodes Corollary~\ref{cor:zero:node} applies and recursively utilizing Corollary~\ref{cor:propagate} and Lemma~\ref{lem:cond:for:propagate} whenever a new edge is identified. In particular, we will refer to the recursive strategy as \emph{propagation}.

The preliminary identification of edges in this manner is a simple and efficient implementation of the aIV strategy:
\begin{proposition}
  Every edge in a tree graph identified by aIV is identified during preliminary identification. 
\end{proposition}

Moreover, we show that, for tree graphs, preliminary identification is at least as effective as the state-of-the-art polynomial-time algorithm ACID.

\begin{proposition}\label{prop:ACID:aIV}
  Every edge in a tree graph identified by the ACID algorithm is
  identified during preliminary identification.
\end{proposition}



In the following sections, we derive an entirely new approach to direct effect identification based on missing cycles of bidirected edges, able to identify even further parameters.

\section{MISSING CYCLE EQUATIONS}\label{sec:equation:system:solution}



In this section, we show how a missing cycle of bidirected edges can yield an identification of directed edges that point at the nodes of the cycle:

\begin{definition}\label{def:polynomials:abcd}
Let $v_1, \ldots, v_k > 0$ be a missing cycle with parents $p_i \to v_i$. Let $v_{k+1} = v_1$ and $p_{k+1} = p_1$. Let $L = \lceil \log_2 k\rceil + 1$.

Define the polynomials $a_i^{(l)}, b_i^{(l)}, c_i^{(l)}, d_i^{(l)}$, for $l=1,\ldots,L$, recursively as follows

$a^{(l+1)}_i = \begin{cases}
\sigma_{p_i,p_{i+1}} & l = 0\\
a^{(l)}_{2i-1} &  \lceil \frac{k}{2^{l-1}} \rceil \text{ is odd} \wedge i = \lceil \frac{k}{2^l} \rceil  \\[2mm]
\det \left( \begin{matrix}
a^{(l)}_{2i-1}&a^{(l)}_{2i}  \\
b^{(l)}_{2i-1}&c^{(l)}_{2i}  
\end{matrix} \right) & \text{else},
\end{cases}$ \\[3mm]
$b^{(l+1)}_i = \begin{cases}
 - \sigma_{p_i,v_{i+1}}& l = 0\\
b^{(l)}_{2i-1} & \lceil \frac{k}{2^{l-1}} \rceil \text{ is odd} \wedge i = \lceil \frac{k}{2^l} \rceil  \\[2mm]
\det \left( \begin{matrix}
a^{(l)}_{2i-1}& b^{(l)}_{2i} \\
b^{(l)}_{2i-1}& d^{(l)}_{2i} 
\end{matrix} \right) & \text{else},
\end{cases}$ \\[3mm]
$c^{(l+1)}_i = \begin{cases}
- \sigma_{v_i,p_{i+1}}& l = 0\\
c^{(l)}_{2i-1} & \lceil \frac{k}{2^{l-1}} \rceil \text{ is odd} \wedge i = \lceil \frac{k}{2^l} \rceil  \\[2mm]
\det \left( \begin{matrix}
 c^{(l)}_{2i-1}&a^{(l)}_{2i} \\
 d^{(l)}_{2i-1}&c^{(l)}_{2i} 
\end{matrix} \right) & \text{else},
\end{cases}$ \\[3mm]
$d^{(l+1)}_i = \begin{cases}
 \sigma_{v_{i},v_{i+1}}  & l = 0\\
d^{(l)}_{2i-1} & \lceil \frac{k}{2^{l-1}} \rceil \text{ is odd} \wedge i = \lceil \frac{k}{2^l} \rceil  \\[2mm]
\det \left( \begin{matrix}
 c^{(l)}_{2i-1} &b^{(l)}_{2i}\\
 d^{(l)}_{2i-1} &d^{(l)}_{2i}
\end{matrix} \right) & \text{else}.
\end{cases}$

\end{definition}

A missing cycle encodes a quadratic equation for each incoming edge that can yield two possible solutions:

\begin{theorem}\label{lem:tree:solution}
Let $\cG=(V,D,B)$ be a tree graph and
assume there is a cycle $v_1, \ldots, v_k > 0$, such that each edge 
$v_i \bidirected v_{i+1}$, with $v_{k+1} = v_1$, is missing. 
%
%
Then the path coefficient $\lambda_{v_1}$ 
satisfies the equation 
\begin{equation}\label{eq:thm:main:eq}
a^{(L)}_1 \lambda_{v_1}^2 + (b^{(L)}_1 + c^{(L)}_1) \lambda_{v_1} + d^{(L)}_1 = 0
\end{equation}
where $a,b,c,d$ are calculated as in Definition~\ref{def:polynomials:abcd}.
%
%
%
%
\end{theorem}

\begin{lemma}\label{lemm:number:of solutions}
In the following cases, the equation~\eqref{eq:thm:main:eq} of Theorem~\ref{lem:tree:solution} has one or two solutions:

If $\evalSigma{a^{(L)}_1}\!\equiv 0 \wedge 
  \evalSigma{b^{(L)}_1 + c^{(L)}_1}\!\not\equiv 0$, then $\IDSCORE[\cG]{\lambda_{v_1}}\!=\!1$.

If $\evalSigma{a^{(L)}_1} \not\equiv  0  \wedge  
  \evalSigma{(b^{(L)}_1 + c^{(L)}_1)^2 - 4 a^{(L)}_1 d^{(L)}_1} \not\equiv 0  $, then $\IDSCORE[\cG]{\lambda_{v_1}} \leq 2$.

If $\evalSigma{a^{(L)}_1} \not\equiv 0 \wedge 
  \evalSigma{(b^{(L)}_1 + c^{(L)}_1)^2 - 4 a^{(L)}_1 d^{(L)}_1} \equiv 0  $, then $\IDSCORE[\cG]{\lambda_{v_1}} = 1$.
\end{lemma}

\begin{proof}
In the first case, the equation becomes linear and has a 
solution $\lambda_{v_1} = \frac{-d^{(L)}_1}{b^{(L)}_1 + c^{(L)}_1}$.

If a solution exists, then it will always be real-valued. Therefore, the polynomial 
$\evalSigma{(b^{(L)}_1 + c^{(L)}_1)^2 - 4 a^{(L)}_1 d^{(L)}_1}$ will always be non-negative.
If it is nonzero, then there are two solutions 
$\lambda_{v_1} = \frac{ -(b^{(L)}_1 + c^{(L)}_1) \pm \sqrt{(b^{(L)}_1 + c^{(L)}_1)^2 - 4 a^{(L)}_1 d^{(L)}_1} }{ 2 a^{(L)}_1 }$. 

If the square root is zero, then there is effectively only one solution.
\end{proof}

If one edge into a missing cycle is identifiable, all other edges into this missing cycle are also identifiable. From equation $\lambda_{pi} \lambda_{qj} \sigma_{pq} - \lambda_{pi} \sigma_{pj} - \lambda_{qj} \sigma_{iq} + \sigma_{ij} = 0$, it follows $ \lambda_{qj} =  (\lambda_{pi} \sigma_{pj}  -  \sigma_{ij}) / ( \lambda_{pi} \sigma_{pq} - \sigma_{iq} )$, so knowing one edge $\lambda_{pi}$, one can usually derive the other edges. However, this might not always be possible since $\evalSigma{\lambda_{pi} \sigma_{pq} - \sigma_{iq}} = 0$ might occur.

%
%
%
%
%
%


\section{THE ALGORITHM} \label{sec:algorithm}
In this section, we present an  algorithm, called TreeID, 
to identify parameters in tree models  $\cG = (V,D,B)$.
We assume $V=\{0,\ldots,n\}$ such that the nodes are numbered in topological order, i.e.\ if 
$i\in \An(j)$, then $i\le j$. Thus, in particular, $0$ is the root node.

The algorithm, presented as Algorithm~\ref{alg:id}, uses the array  $\mbox{\rm ID}[1,\ldots,n]$
to store the solutions for parameters $\lambda_i$ for edges $p\to i \in D$, as functions over $\sigma_{jk}$.
Initially, all $\mbox{\rm ID}[i]=\emptyset$  meaning the parameter is not-identified. 
At the end of the algorithm, if $\mbox{\rm |ID}[i]|=1$, then $\lambda_i$ 
is  identifiable and given by the formula in $\mbox{\rm ID}[i]$. If $\mbox{\rm |ID}[i]|=2$,
then $\lambda_i$ is identifiable by at least one of the solutions given in $\mbox{\rm ID}[i]$.
During its work, the algorithm represents the formulas  in  $\mbox{\rm ID}[i]$ in the 
FASTP form $\frac{p+q\sqrt{s}}{r+t\sqrt{s}}$ where $p,q,r,s,t$ are polynomials over $\sigma_{jk}$.

TreeID starts with the identification of $\lambda_i$ for each $i$, such that 
$0 \bidirected i \notin B$. To this aim node $0$ is used as an instrumental 
variable (Corollary~\ref{cor:zero:node}). Based on Corollary~$\ref{cor:propagate}$,
the identification for $\lambda_i$ is ``propagated'' (recursively)
to identify parameters $\lambda_j$, with  
$i \bidirected j \notin B$, $p\to i, q\to j \in D$,
as $\lambda_{j}  = ( \lambda_{i} \sigma_{pj}  - \sigma_{ij} ) / ( \lambda_{i} \sigma_{pq} - \sigma_{iq})$,
if the function in the denominator is non-zero.

The main part of the algorithm identifies the parameters $\lambda_i$, 
which have not been recognized as identifiable in the initial phase.
To this aim, for each such $i$, TreeID proceeds as follows:  For every 
``missing'' cycle $v_1=i, v_2, \ldots, v_k > 0$ including $i$, i.e.\ for a sequence of nodes
such that  $v_j \bidirected v_{j+1} \notin B$ for all $j=1,\ldots k$, with $v_{k+1} = v_1$,
the algorithm computes a quadratic equation $a\lambda^2_i + b\lambda_i+c = 0$
using Theorem~\ref{lem:tree:solution}.
If both $\evalSigma{a} \equiv 0$ and $\evalSigma{b} \equiv 0$,
then all $\lambda_i$ satisfy the equation and thus 
the algorithm skips the cycle.
If only $\evalSigma{a} \equiv 0$, the equation has exactly one solution, which is stored in $\mbox{\rm ID}[i]$.
Otherwise, if $\lambda_i$ is not yet identified ($\mbox{\rm ID}[i]=\emptyset$),
the algorithm using Lemma~\ref{lemm:number:of solutions} computes one or two solutions for $\lambda_i$;
Otherwise, it updates the solutions $\mbox{\rm ID}[i]$ calculated so far, by removing from 
the set such $\lambda$'s that do not satisfy the equation $a\lambda^2_i + b\lambda_i+c = 0$.
Finally, similarly as in the initial phase, the identifications for $\lambda_i$ are  propagated
to compute formulas for  $\lambda_j$, with $i \bidirected j \notin B$.   
This proves the following:
\begin{theorem}\label{thm:algorithm}
The identification algorithm~TreeID is sound for tree graphs, that is, 
for  a given tree $\cG=(V,D,B)$, with $V=\{0,\ldots,n\}$, if it 
returns $i$ and $\mbox{\rm ID}[i]$, then $\lambda_i$  
is  identifiable and given by the formula.
Additionally, if at the end of the algorithm $\mbox{\rm |ID}[i]|=2$, then $\lambda_i$ is identifiable with at least one of the solutions given in $\mbox{\rm ID}[i]$.
\end{theorem}

\begin{algorithm}[caption={TreeID}, label={alg:id}, texcl]
input: $\text{tree graph}\ \cG = (V=\{0,\ldots,n\},D,B)$
output: $\text{a set of identifiable structural parameters }$

  function SolveEquation()
    input: $a\lambda^2 + b\lambda+c=0$

    if $\evalSigma{b^2 - 4ac} \equiv 0$:  return $\{-b / 2a\}$
    $s \gets \sqrt{b^2 - 4ac}$
    return $\{(-b - s)/2a, (-b + s)/2a\}$
    
  function Propagate()   //use Corollary $\ref{cor:propagate}$
    input: $i$    
    $p \gets \Pa(i)$
    for each $i\bidirected j \not\in B$:
      if $0<$|ID$[j]$|$\leq$|ID$[i]$|: continue
      $q \gets \Pa(j)$
      if $\exists \lambda\in\texttt{ID}[i]\ s.t. \evalSigma{(\lambda  \sigma_{pq}  - \sigma_{iq})} \equiv 0$: continue
      ID$[j] \gets \{  (\lambda \sigma_{pj} - \sigma_{ij} ) / (\lambda \sigma_{pq}  - \sigma_{iq})\mid \lambda \in \texttt{ID}[i]  \}$ 
      Propagate($j$)
             
for $i$ $\gets$ 1 $\ldots$ $n$:
  ID$[i]\gets \emptyset$      //mark all nodes as not identified
  
for $i$ $\gets$ 1 $\ldots$ $n$:
  if $0 \bidirected i \notin B$:
    ID$[i]$ $\gets$ $\{\sigma_{0i} / \sigma_{0p}\}$  //use Corollary $\ref{cor:zero:node}$
    Propagate($i$)    
    
for $i$ $\gets$ 1 $\ldots$ $n$:
  if |ID$[i]$|$= 1$: continue
  for each missing cycle involving node $i$:
    Use Thm. $\ref{lem:tree:solution}$ to get a quadratic equation 
          $a\lambda^2_i + b\lambda_i+c = 0$
    if $\evalSigma{a} \equiv 0$ and $\evalSigma{b} \equiv 0$: continue
  
    if $\evalSigma{a} \equiv 0$: 
      ID$[i]$ $\gets$ $\{-c / b\}$
    else if ID$[i]=\emptyset$: 
      ID$[i]$ $\gets$ SolveEquation($a\lambda^2_i + b\lambda_i+c=0$)
    else 
      ID$[i]$ $\gets$ $\{\lambda \in \texttt{ID}[i] \mid \evalSigma{ a\lambda^2 + b\lambda+c} \equiv 0  \}$
    
    if |ID$[i]$|$= 1$: break
  Propagate($i$)
      
for $i$ $\gets$ 1 $\ldots$ $n$:
  if $\mbox{\rm |ID}[i]|=1$: return $i,$ ID$[i]$ 
\end{algorithm}


The tests $\evalSigma{F} \equiv 0$ for FASTPs $F$ over $\sigma_{jk}$ can be reduced to PITs according to Lemma~\ref{lemm:FASTP}. The standard algorithm for deciding PIT runs in randomized polynomial time using Lemma~\ref{lemm:Schwartz:Zippel}. It is a blackbox algorithm, i.e.\ it does not require a representation of the polynomial, only the evaluated value of the polynomial. Although the calculation of the FASTP in the propagation step can double the size of the polynomials, TreeID can store the constant size equation $\lambda_j = ( \lambda_{i} \sigma_{pj}  - \sigma_{ij} ) / ( \lambda_{i} \sigma_{pq} - \sigma_{iq})$ directly without expanding $\lambda_{i}$. For PIT, we then evaluate all $\lambda_i, \lambda_j$ recursively, storing the value of the four polynomials in the FASTP separately. 
This gives:

\begin{proposition}\label{prop:TreeID:compl}
  For a given tree graph $\cG$, the  running time of TreeID algorithm is
  in $O(p(n) \cdot \text{mc}_{\cG})$ randomized time, where $p(n)$ is a polynomial for solving PIT and 
  $\text{mc}_{\cG}$ denotes the number of missing (bidirectional) cycles in $\cG$.
\end{proposition}

One can also see that the algorithm runs in polynomial space (PSPACE).

%
%

\begin{proposition}\label{prop:TreeID:ACID}
If $\lambda_i$ is identifiable with the 
ACID algorithm (that covers cAV, IC, qAVS criteria), 
then it is identifiable with the TreeID algorithm.
\end{proposition}

\section{EXAMPLES}\label{sec:examples}

In this section we first  explain how the algorithm TreeID  identifies 
the graphs in Fig.~\ref{fig:examples}. 
Next, we show how our algorithm works on instances considered in 
\citet{identifyingEdgeWiseDeterminantalDrton} which 
are \emph{unidentifiable} by HTC and the TSID algorithm. 
Finally, we discuss path graphs -- as a special case of tree graphs,
whose bidirected component is complete except for exactly one missing cycle.

\catcode42=9
\def\sqrtvar{s}

\subsection{Models in Figure~\ref{fig:examples}}
In the classic IV model $\cG_1$, the root node has no bidirected edges. Thus TreeID immediately identifies all causal effects as $\lambda_1 = \frac{\sigma_{01}}{\sigma_{00}}$ and  $\lambda_2 = \frac{\sigma_{02}}{\sigma_{01}}$ using Corollary~\ref{cor:zero:node}.

In $\cG_2$, the root node is connected to all other nodes by bidirected edges, so Corollary~\ref{cor:zero:node} cannot be applied. TreeID then proceeds to search missing cycles.
One such cycle is $1 \bidirected 2 \bidirected 3 \bidirected 4 \bidirected 1$. 
Applying Theorem~\ref{lem:tree:solution} to this cycle and simplifying the polynomials, yields a quadratic equation with

$a = 0$

$b = 	(σ12*σ13-σ11*σ23)*(σ15*σ44-σ14*σ45)\\
\phantom{b = 	}-(σ14*σ25-σ15*σ24)*(σ11*σ34-σ13*σ14)$

$c = (σ12*σ13-σ11*σ23)*(σ24*σ45-σ25*σ44)\\
\phantom{b = 	}-(σ14*σ25-σ15*σ24)*(σ14*σ23-σ12*σ34)$

Thus $\lambda_{1}$ is identified as $- c / b$. Every other effect $\lambda_j$ of edge $q\to j$ can  then be identified by propagation, 
 $\lambda_{j}  = ( \lambda_{1} \sigma_{0j}  - \sigma_{1j} ) / ( \lambda_{1} \sigma_{0q} - \sigma_{1q})$ using Corollary~\ref{cor:propagate}. Once the algorithm has found a solution for each edge, it is finished.

Since we did not specify an order of the cycles in the pseudocode of the algorithm, it might start with other cycles. If it first finds the missing cycle $1\bidirected  2\bidirected  3\bidirected  1$ and applies Theorem~\ref{lem:tree:solution} there, it obtains a quadratic equation with $a\neq 0$ and two possible solutions for $\lambda_1$ involving a square root. It then has to continue searching cycles, and might find 
 $1\bidirected  3\bidirected  4\bidirected  1$. Only one of the two previous solutions is also a solution for this cycle, so the algorithm eliminates one of them. The remaining solution for $\lambda_1$ then helps again to identify all other edges through propagation.

In $\cG_3$, there are three missing cycles $1\bidirected 2\bidirected 3$, $2\bidirected 3 \bidirected 4$, and $1 \bidirected 2 \bidirected 3 \bidirected 4$.
The missing cycle $1\bidirected 2\bidirected 3$ yields a quadratic equation with two solutions:

$\lambda_1=(\sqrt\sqrtvar+(σ12*σ23+σ13*σ22)*σ34+(-σ12*σ24-σ14*σ22)*σ33-σ13*σ23*σ24+σ14*σ23^2)/(2*σ12*σ13*σ34-2*σ12*σ14*σ33-2*σ13^2*σ24+2*σ13*σ14*σ23)$, and

$\lambda'_1=(-\sqrt\sqrtvar+(σ12*σ23+σ13*σ22)*σ34+(-σ12*σ24-σ14*σ22)*σ33-σ13*σ23*σ24+σ14*σ23^2)/(2*σ12*σ13*σ34-2*σ12*σ14*σ33-2*σ13^2*σ24+2*σ13*σ14*σ23)$,

where $\sqrtvar = (σ12^2*σ23^2-2*σ12*σ13*σ22*σ23+σ13^2*σ22^2)*σ34^2+(((2*σ12*σ13*σ22-2*σ12^2*σ23)*σ24+2*σ12*σ14*σ22*σ23-2*σ13*σ14*σ22^2)*σ33+(2*σ13^2*σ22*σ23-2*σ12*σ13*σ23^2)*σ24+2*σ12*σ14*σ23^3-2*σ13*σ14*σ22*σ23^2)*σ34+(σ12^2*σ24^2-2*σ12*σ14*σ22*σ24+σ14^2*σ22^2)*σ33^2+((2*σ12*σ13*σ23-4*σ13^2*σ22)*σ24^2+(6*σ13*σ14*σ22*σ23-2*σ12*σ14*σ23^2)*σ24-2*σ14^2*σ22*σ23^2)*σ33+σ13^2*σ23^2*σ24^2-2*σ13*σ14*σ23^3*σ24+σ14^2*σ23^4$.

The solutions are distinct because $\evalSigma{\sqrt\sqrtvar}$ simplifies to non-zero $\sqrt{{{\left( {\lambda_{01}} {\omega_{01}}+\omega_{11}\right) }^{2}} {{\left( 2 {\lambda_{01}} {\lambda_{12}} {\omega_{02}}+\omega_{22}\right) }^{2}} {{{\omega_{03}}}^{2}}}$. 

Since there are two solutions, the algorithm continues searching cycles. It might find $1\bidirected 2\bidirected3 \bidirected 4$ next and conclude that the first solution $\lambda_1$ is the only solution. Knowing one solution, it can identify all other edges through propagation.


\catcode42=12

More details are given in the supplementary material. 

\subsection{Graphs unidentifiable by HTC and TSID} 
\citet{identifyingEdgeWiseDeterminantalDrton} have investigated the identifiability of all graphs with 5 nodes. There are 53 graphs in which each edge is uniquely identifiable using Gröbner bases, but that cannot be identified with the halftrek or TSID algorithm. Of these graphs, 15 are acyclic and 5 of those are trees.  We show these trees in Fig.~\ref{fig:examples:drton} and have applied our algorithm to them. There are bidirected edges from the root node to every other node, so the IV method cannot be used.



\begin{figure*}
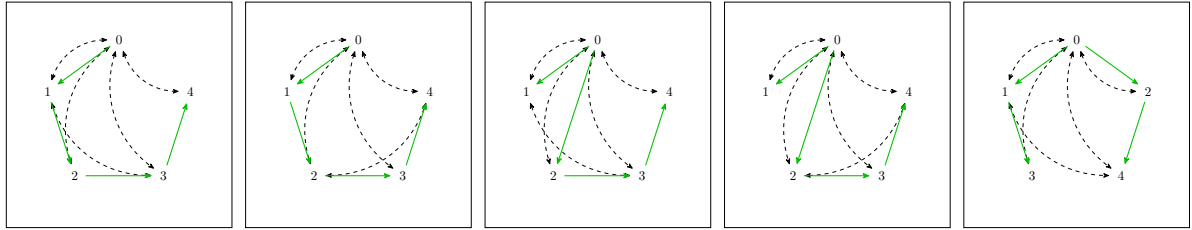

\scalebox{.5}{
\makegraph{\edge{0}{1}\edge{0}{2}\edge{0}{3}\edge{0}{4}\edge{1}{3}\markI{4}\markZ{1}\markZ{2}\markZ{3}}
\makegraph{\edge{0}{1}\edge{0}{2}\edge{0}{3}\edge{0}{4}\edge{2}{4}\markI{4}\markZ{1}\markZ{2}\markZ{3}}
\makegraph{\edge{0}{1}\edge{0}{2}\edge{0}{3}\edge{0}{4}\edge{1}{3}\markZ{1}\markctreeedge[green!75!black]{0}{2}\markZ{3}\markI{4}}
\makegraph{\edge{0}{1}\edge{0}{2}\edge{0}{3}\edge{0}{4}\edge{2}{4}\markZ{1}\markctreeedge[green!75!black]{0}{2}\markZ{3}\markI{4}}
\placeFiveTreeNodes
\makegraph{\edge{0}{1}\edge{0}{2}\edge{0}{3}\edge{0}{4}\edge{1}{4}%
\markI{1}\markctreeedge[green!75!black]{1}{3}\markctreeedge[green!75!black]{0}{2}\markctreeedge[green!75!black]{2}{4}%
}
} 

\caption{Tree graphs of \citep{identifyingEdgeWiseDeterminantalDrton}, where each directed edge is identifiable (green). The authors name them
 (4680, 403),  (4680, 914), (360, 117), (360, 369), (840, 466). %
We relabel the nodes to make node 0 the root. 
}
\label{fig:examples:drton}
\end{figure*}

In the first graph, there are three missing cycles, $1\bidirected 2\bidirected 4$, $2\bidirected 3\bidirected 4$, and $1\bidirected 2\bidirected 3\bidirected 4$. Each cycle yields two solutions. 

\catcode42=9

\def\sqrtvar{s}

The solutions for $\lambda_1$ of $1\bidirected 2 \bidirected 4$ are\\
$\lambda_1=(\sqrt{\sqrtvar}+(σ12*σ24+σ14*σ22)*σ35+(-σ12*σ25-σ15*σ22)*σ34-σ14*σ23*σ25+σ15*σ23*σ24)/(2*σ12*σ14*σ35-2*σ12*σ15*σ34-2*σ13*σ14*σ25+2*σ13*σ15*σ24)$ and \\
$\lambda'_1=(-\sqrt{\sqrtvar}+(σ12*σ24+σ14*σ22)*σ35+(-σ12*σ25-σ15*σ22)*σ34-σ14*σ23*σ25+σ15*σ23*σ24)/(2*σ12*σ14*σ35-2*σ12*σ15*σ34-2*σ13*σ14*σ25+2*σ13*σ15*σ24)$\\
where 
$\sqrtvar = (-σ14*(σ23*σ25-σ22*σ35)+σ24*(σ12*σ35-σ13*σ25)-σ15*(σ22*σ34-σ23*σ24)+σ25*(σ13*σ24-σ12*σ34))^2-4*(-σ14*(σ12*σ35-σ13*σ25)-σ15*(σ13*σ24-σ12*σ34))*(σ24*(σ23*σ25-σ22*σ35)+σ25*(σ22*σ34-σ23*σ24))$.



The solutions are distinct because $\evalSigma{s}$ simplifies to non-zero 
${{{\omega_{04}}}^{2}} ( {\lambda_{01}} {\omega_{02}} {\omega_{13}}+{\lambda_{01}} {{{\lambda_{12}}}^{2}} {\lambda_{23}} {\omega_{01}}-{\lambda_{01}} {\lambda_{23}} {\omega_{01}}+{{{\lambda_{01}}}^{2}} {{{\lambda_{12}}}^{2}} {\lambda_{23}}-{{{\lambda_{12}}}^{2}} {\lambda_{23}}-{{{\lambda_{01}}}^{2}} {\lambda_{23}}+{\lambda_{23}})^{2}$.

The former $\lambda_1$ is also a solution for the cycle  $1\bidirected 2\bidirected 3\bidirected 4$. The latter $\lambda'_1$ is not. Thus $\lambda_1$ is the true solution. 
Using propagate, all other edges are identifiable as well.

The same situation occurs in the other graphs. The individual cycles have two solutions, and the combination of a 3-cycle with a 4-cycle yields exactly one solution. The second graph is $\cG_3$ in Figure~\ref{fig:examples}. Further solutions are given in the supplementary material.

\subsection{Path Graphs with a Single Missing Cycle}\label{sec:pat:graphs}
\catcode42=12
\def\numberednode#1{#1}

\def\thenodes{\circlednode{0}\circlednode{1}\circlednode{2}\circlednode{3}\circlednode{4}\circlednode{5}\circlednode{6}\circlednode{7}} 
\def\thenonidedges{\markFail{1}\markFail{2}\markFail{3}\markFail{4}\markFail{5}\markFail{6}\markFail{7}}


\def\somenonidedges{\markFail{1}\markFail{2}\markFail{3}\markFail{4}}
\begin{figure*}[h]
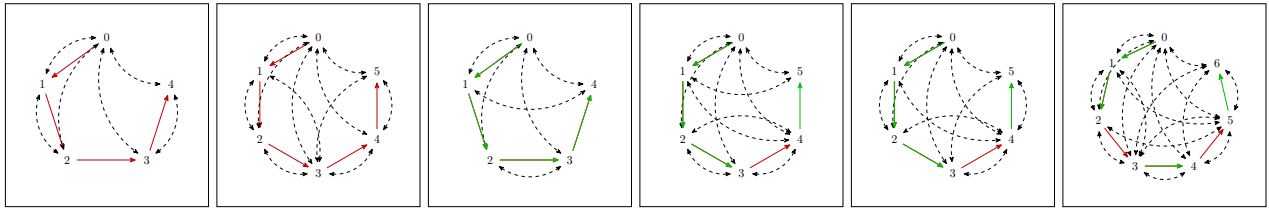

\scalebox{.45}{
\placeFiveNodes%
\makegraph{\somenonidedges\edge{0}{1}\edge{0}{2}\edge{0}{3}\edge{0}{4}%
\edge{1}{2}\edge{3}{4}}
\placeSixNodes%
\makegraph{\somenonidedges\edge{0}{1}\edge{0}{2}\edge{0}{3}\edge{0}{4}\edge{0}{5}%
\edge{1}{2}\edge{3}{1}\edge{2}{3}\edge{3}{4}\edge{5}{3}\edge{4}{5}\markFail{5}}
\placeFiveNodes%
\makegraph{\somenonidedges\edge{0}{1}\edge{0}{2}\edge{0}{3}\edge{0}{4}%
\edge{1}{4}\edge{2}{3}\markI{1}\markI{2}\markI{3}\markI{4}}
\placeSixNodes%
\makegraph{\somenonidedges\edge{0}{1}\edge{0}{2}\edge{0}{3}\edge{0}{4}\edge{0}{5}%
\edge{1}{4}\edge{1}{5}\edge{2}{3}\edge{4}{2}\edge{3}{4}\markI{1}\markI{2}\markI{3}\markI{5}}
\makegraph{\somenonidedges\edge{0}{1}\edge{0}{2}\edge{0}{3}\edge{0}{4}\edge{0}{5}%
\edge{1}{2}\edge{1}{4}\edge{4}{2}\edge{3}{4}\edge{5}{3}\edge{4}{5}\markI{1}\markI{2}\markI{3}\markI{5}}
\placeSevenNodes%
\makegraph{\somenonidedges\edge{0}{1}\edge{0}{2}\edge{0}{3}\edge{0}{4}\edge{0}{5}\edge{0}{6}%
\edge{1}{2}\edge{3}{1}\edge{1}{5}\edge{2}{3}\edge{2}{5}\edge{3}{4}\edge{5}{3}\edge{6}{3}\edge{4}{5}\edge{6}{4}\edge{5}{6}\markI{1}\markI{2}\markI{4}\markI{6}\markFail{5}}
}
\caption{The two unidentifiable and the four uniquely identifiable path graphs with a single missing cycle. 
Identifiable edges are shown in green and unidentifiable edges in red.}
\label{fig:special:4:cycles}
\end{figure*}

Now we consider path graphs, whose bidirected component is complete except for exactly one missing cycle. 
This is an important class to investigate to understand how many solutions 
Theorem~\ref{lem:tree:solution} can yield.
%

First we show a way of transforming graphs into equivalent graphs. Since only the covariances between nodes with missing bidirected edges and their parents 
occur in the equations of Lemma~\ref{lem:reversed:eqs:tree}, all other edges and nodes can be removed, added, or permutated without changing the solutions:


\def\nodeidswithnobidiedge{\bm}

\begin{lemma}\label{lem:path:transformation}
Let $\cG=(V,D,B)$ be a path graph with nodes $\numberednode 0,\ldots,\numberednode n$ and $0 \bidirected i \in B$ for all $i\in\{1,\ldots,n\}$. Let $\nodeidswithnobidiedge  = \{i \mid \exists j (  i \bidirected \numberednode j \notin B \vee \numberednode {i+1} \bidirected \numberednode {j} \notin B) \wedge i, j > 0  \}$ be the nodes on directed edges into the missing cycle.
The identifiability does not change if:
\begin{itemize}
\item Nodes not in $\bm \cup 0$ are removed from the graph.
\item The nodes are permutated by a permutation $\pi$ with $\pi(0) = 0$; and,
for all $i \in \nodeidswithnobidiedge$ and $i + 1 \in \nodeidswithnobidiedge$, $\pi(i + 1) = \pi(i) + 1$.
\end{itemize}

\end{lemma}

Thus for each missing cycle path graph, there is a canonical graph, which can be obtained by permutating all nodes not affecting the missing cycle to the end of the path and removing them.
For a certain cycle length, there are only finitely many canonical graphs.
%
%
Enumerating all canonical graphs up to length 10, shows there are only 6 missing cycles whose incoming edges are not 2-identifiable (see Figure~\ref{fig:special:4:cycles}). 
The first two are not identifiable at all, the other four have a unique solution for all edges into the missing cycle:

\begin{compactitem}
\item $1 \bidirected 3  \bidirected 2 \bidirected 4 \bidirected 1$ (unidentifiable) 
\item $1 \bidirected 4  \bidirected 2 \bidirected 5 \bidirected 1$ (unidentifiable) 
\item $1 \bidirected 2  \bidirected 4 \bidirected 3 \bidirected 1$ (uniquely identifiable)
\item $1 \bidirected 2  \bidirected 5 \bidirected 3 \bidirected 1$ (uniquely identifiable)
\item $1 \bidirected 3  \bidirected 2 \bidirected 5 \bidirected 1$ (uniquely identifiable)
\item $1 \bidirected 4  \bidirected 2 \bidirected 6 \bidirected 1$ (uniquely identifiable)
\end{compactitem}



\begin{proposition}
Any path graph with exactly one missing cycle that is not equivalent (after transformations with Lemma~\ref{lem:path:transformation}) 
to one of the graphs of Fig.~\ref{fig:special:4:cycles} is 2-identifiable, if the cycle length is at most~10.
\end{proposition}
%
%

%
%
% ...

%
%

The above results have been obtained using Gröbner bases as a reference solution. This enumeration was only possible because the Gröbner base calculation performed vastly faster on path graphs with a single missing cycle than on path graphs with arbitrarily missing bidirected edges.

Our algorithm can be applied to the graphs in Figure~\ref{fig:special:4:cycles} and Theorem~\ref{lem:tree:solution} returns quadratic equations for the missing cycles. For the unidentifiable graphs, the quadratic equation vanishes that is $\evalSigma{a} = \evalSigma{b + c} = \evalSigma{d} = 0$. For the identifiable graphs, there is one edge for which the theorem returns a linear equation, i.e.\ $\evalSigma{a} = 0$, which implies that all incoming edges are identifiable with propagation.
The exact results are given in the supplementary material.

\section{CONCLUSIONS}\label{sec:concl}
Our algorithm allows the identification of causal effects in tree graphs which could not be identified previously without Gröbner bases.
It is possible that the algorithm is complete for the considered family of causal models, i.e., it is able to identify all causal effects in tree graphs if and only if the effects are identifiable by any method, although we cannot prove that.
However, to our knowledge, no algorithm is known so far that is complete for a certain natural family of graphs.

In the worst case, the algorithm enumerates all missing cycles, although that is unnecessary since two cycles for each edge are  always sufficient to identify an identifiable edge. An open problem remains to find just those two cycles efficiently.
The next issue is that the 
solutions provided by the algorithm are in the form of fairly  complex expressions, so their numerical stability on real life datasets needs  further investigation. 

Future research might generalize the algorithm to further graph classes.
The missing cycle method of Theorem 1 only requires that all nodes on the missing cycle have one and
only one incoming directed edge. If the resulting equation systems satisfy the conditions of Lemma 2, our algorithm can probably already be used to identify the incoming edges in more complex graphs that only contain a tree as subgraph.

We have implemented our algorithm for the DAGitty
project~({\tt www.dagitty.net}, see \cite{dagittyIJE}). 
To ensure its correctness, we have compared this implementation with a Gröbner base implementation (see supplementary material \ref{sec:dagitty:vs:gröbner}).

\paragraph{Acknowledgments}

The first author 
gratefully acknowledges the financial support by the Federal Ministry for Economic Affairs and Climate Action of Germany (BMWK) through the KI-SIGS Project (FKZ: 01MK20012B).
This work was also supported by the Deutsche Forschungsgemeinschaft (DFG) grant 471183316 (ZA 1244/1-1).

\bibliographystyle{abbrvnat}
\bibliography{all}


\clearpage
\appendix

\thispagestyle{empty}

\onecolumn \makesupplementtitle

\section{MISSING PROOFS}

\subsection{Proof of Lemma~\ref{lemm:exp:for:sigma}}

\begin{proof}
According to the trek rule~\eqref{eqn:trek:rule} 
we know that $\sigma_{ij}$ is given by the sum of the products along a trek over all
treks  between $i$ and $j$. 
A trek containing the bidirected edge $s\bidirected t$ contributes $\omega_{st} L(s,i) L(t, j)$. 
A trek without a  bidirected edge contributes $\omega_{ss} L(s,i) L(s, j)$, for 
$s=t\in \An(i)\cap \An(j)$.
\end{proof}

\subsection{Proof of Lemma~\ref{lem:tree:reversed:equations:new}}

\begin{proof}
For an arbitrary mixed graph and the corresponding matrix
$
\Omega = (I - \Lambda)^{T} \Sigma (I-\Lambda) 
$,
\citeauthor{drton2018algebraic}  shows  (\citeyear{drton2018algebraic})
that
$[(I - \Lambda)^{T} \Sigma (I-\Lambda) ]_{ij} =$
 $ \sum_{p\in \Pa(i)} \sum_{q\in \Pa(j)} \lambda_{pi} \lambda_{qj} \sigma_{pq} 
  - \sum_{p\in \Pa(i)} \lambda_{pi}\sigma_{pj} 
  - \sum_{q\in \Pa(j)} \lambda_{qj}\sigma_{pi}
  +  \sigma_{ij} .
  $
Our lemma for tree graphs follows from the equation.  
\end{proof}

An alternate proof can be given directly using Wright's path rules.

\begin{proof} We show first that $ \evalSigma{\lambda_{pi} \lambda_{qj} \sigma_{pq} - \lambda_{pi} \sigma_{pj} - \lambda_{qj} \sigma_{iq} + \sigma_{ij} } = \omega_{ij}$. Indeed,

\begin{align*}
\evalSigma{\lambda_{p,i} \lambda_{q,j} \sigma_{p,q} - \lambda_{p,i} \sigma_{p,j} - \lambda_{q,j} \sigma_{i,q} + \sigma_{i,j}}\\
%
\end{align*}

=

\begin{align*}
\lambda_{p,i} \lambda_{q,j} \left(\sum_{s\in \An(p)} \sum_{t\in\An(q)} \omega_{s,t} L(s,p) L(t, q) \right)\\
- \lambda_{p,i} \left(\sum_{s\in \An(p)} \sum_{t\in\An(j)} \omega_{s,t} L(s,p) L(t, j) \right)\\
- \lambda_{q,j} \left( \sum_{s\in \An(i)} \sum_{t\in\An(q)} \omega_{s,t} L(s,i) L(t, q) \right)\\
+ \sum_{s\in \An(i)} \sum_{t\in\An(j)} \omega_{s,t} L(s,i) L(t, j)\\
\end{align*}

=

\newpage

\begin{align*}
\left(\sum_{s\in \An(p)} \sum_{t\in\An(q)} \omega_{s,t} L(s,p) \lambda_{p,i} L(t, q) \lambda_{q,j} \right)\\
- \left(\sum_{s\in \An(p)} \sum_{t\in\An(j)} \omega_{s,t} L(s,p) \lambda_{p,i}  L(t, j) \right)\\
- \left(\sum_{s\in \An(i)} \sum_{t\in\An(q)} \omega_{s,t} L(s,i) L(t, q) \lambda_{q,j}  \right)\\
+ \sum_{s\in \An(i)} \sum_{t\in\An(j)} \omega_{s,t} L(s,i) L(t, j)\\
\end{align*}

%

=

\begin{align*}
\sum_{s\in \An(p)} \sum_{t\in\An(q)} \omega_{s,t} L(s,i)  L(t, j)  \\
- \sum_{s\in \An(p)} \sum_{t\in\An(j)} \omega_{s,t} L(s,i)  L(t, j) \\
- \sum_{s\in \An(i)} \sum_{t\in\An(q)} \omega_{s,t} L(s,i) L(t, j) \\
+ \sum_{s\in \An(i)} \sum_{t\in\An(j)} \omega_{s,t} L(s,i) L(t, j)\\
\end{align*}

=

\begin{align*}
\sum_{s\in \An(p)} \sum_{t\in\An(q)} \omega_{s,t} L(s,i)  L(t, j)  \\
- \sum_{s\in \An(p)} \sum_{t\in\An(q)\cup j} \omega_{s,t} L(s,i)  L(t, j) \\
- \sum_{s\in \An(p)\cup i} \sum_{t\in\An(q)} \omega_{s,t} L(s,i) L(t, j) \\
+ \sum_{s\in \An(p)\cup i} \sum_{t\in\An(q)\cup j} \omega_{s,t} L(s,i) L(t, j)\\
\end{align*}

=

\begin{align*}
- \sum_{s\in \An(p)} \omega_{s,j} L(s,i)  L(j, j)  \\
+ \sum_{s\in \An(p)\cup i} \omega_{s,j} L(s,i) L(j, j)\\
\end{align*}

= $\omega_{i,j} L(i,i) L(j, j)$ = $\omega_{i,j}$.\\

Now, we show that $  \evalSigma{\sigma_{0i} - \lambda_{qi} \sigma_{0q}} = \omega_{0i}$:

$$\evalSigma{\sigma_{0,i} - \lambda_{p,i} \sigma_{0,p}}$$

= 

\begin{align*}
 \sum_{s\in \An(0)} \sum_{t\in\An(i)} \omega_{s,t} L(s,0) L(t, i)\\
- \lambda_{p,i} \left(\sum_{s\in \An(0)} \sum_{t\in\An(p)} \omega_{s,t} L(s,0) L(t, p)\right)
\end{align*}

= 

\begin{align*}
\sum_{s\in \An(0)} \sum_{t\in\An(i)} \omega_{s,t} L(s,0) L(t, i)\\
- \left(\sum_{s\in \An(0)} \sum_{t\in\An(p)} \omega_{s,t} L(s,0) L(t, p) \lambda_{p,i}\right)
\end{align*}

= 

\begin{align*}
\sum_{s\in \An(0)} \sum_{t\in\An(i)} \omega_{s,t} L(s,0) L(t, i)\\
 - \sum_{s\in \An(0)} \sum_{t\in\An(p)} \omega_{s,t} L(s,0) L(t, i)
\end{align*}

%

=

\begin{align*}
 \sum_{t\in\An(p)\cup i} \omega_{0,t} L(0,0) L(t, i)\\
- \sum_{t\in\An(p)} \omega_{0,t} L(0,0) L(t, i)
\end{align*}

= $ \omega_{0,i} L(0,0) L(i, i)$ = $\omega_{0,i}$.
\end{proof}

\subsection{Proof of Lemma~\ref{lem:reversed:eqs:tree}}

\begin{proof}
If edge $i \bidirected j$ is missing, then $\omega_{ij} = 0$. So from 
Lemma~\ref{lem:tree:reversed:equations:new},
it follows immediately, that the degree of identifiability is at most the number of generic solutions of this equation system.

\citet{halftrek2012} have proven that the degree of identifiability is given exactly by number of solutions of the system $(I - \Lambda) \Sigma (I-\Lambda)^{T} = \Omega $ on the elements corresponding to missing bidirected edges. 
If one calculates the matrix $(I - \Lambda) \Sigma (I-\Lambda)^{T}$, one obtains this equation system.
\end{proof}

\subsection{Proof of Proposition 1}

\begin{proof}
  In the preliminary identification, the root is used as instrumental
  variable (Corollary~\ref{cor:zero:node}) for identifying the edge $x \rightarrow y$,
  in case there is no edge $0 \bidirected y$. It can be easily seen
  that, if this criterion does not apply, there can be no instrument $z$ for
  $x \rightarrow y$ as there is an open backdoor path from $z$ via the
  root to $y$. Hence, using Corollary~\ref{cor:zero:node} is equivalent to IV.

  Whenever an edge is identified, aIV creates an auxiliary variable to
  use as instrument. Propagation (Lemma~\ref{lem:cond:for:propagate}) applies the same
  criterion by ensuring that there is a path from instrument $z$ via
  $x$ to $y$ in the graph with $\rightarrow z$ removed and no backdoor
  path (which again holds iff there is no bidirected edge between $z$
  and $y$).
\end{proof}

\subsection{Proof of Proposition~\ref{prop:ACID:aIV}}

 \begin{proof}
   ACID repeatedly identifies direct effects (and partial effects) and creates corresponding auxiliary
   variables. For this, two criteria are used, which both collapse to
   the IV setting on tree graphs as we show below. Hence, in each iteration, the same
   direct effects are identified and the same auxiliary variables
   created as in aIV. It follows that the whole
   method is equivalent to aIV on tree graphs.

   We now consider the two criteria for the
   identification of vertex $y$ with parent $x$.
   \begin{enumerate}
   \item The instrumental cutset (IC) criterion (originally proposed 
   by~\cite[Theorem~5.1]{kumor2019instrumentalCutsets}):

   As each node has in-degree 1, we have that
     $|S| = 1$ and $|T| = 0$. Any $s \in S$ fulfilling the three IC
     criteria is an instrument for $x \rightarrow y$.
   \item The auxiliary cutset (AC) criterion~\cite[Definition 3.2]{kumor2020auxiliaryCutsets}:
   The auxiliary cutset is defined as
     the closest cutset to $x = \Pa(y)$. In graphs with in-degree 1, it is
     always $x$ itself. Then, a set $Z$ which acts as partial-effect
     instrumental set (PEIS), see~\cite[Definition 3.1]{kumor2020auxiliaryCutsets}, has cardinality 1 and $z \in Z$ is
     an instrumental variable.
   \end{enumerate}
   This shows that, restricted to tree graphs,  
   every edge identified by ACID is identified by  aIV.
   Thus from Proposition 1 the claim of this proposition follows.
 \end{proof}

\subsection{Proof of Theorem \ref{lem:tree:solution}}
\def\eqnvar{x}
\begin{proof}

Due to Lemma~\ref{lem:reversed:eqs:tree}, the missing cycle yields $k$ equations, given by 
$a^1_i \lambda_{v_i} \lambda_{v_{i+1}} + b^1_i \lambda_{v_i} + c^1_i \lambda_{v_{i+1}} + d^1_i = 0$.

We can eliminate every other equation by combining pairs of equations. This is a general approach which works for all equation systems of this structure 
even if they do not come from a missing cycle. 

Let, for short,  $\eqnvar_i = \lambda_{v_i}$ and $\eqnvar_{k+1}=\eqnvar_1$. We combine the $(2i-1)$th with the $(2i)$th equation:
\begin{align*}
0 =\ &(\eqnvar_{2i-1} a_{2i-1} + c_{2i-1} ) (a_{2i} \eqnvar_{2i} \eqnvar_{2i+1} + b_{2i} \eqnvar_{2i} + c_{2i} \eqnvar_{2i+1} + d_{2i}) \\
&-(\eqnvar_{2i+1} a_{2i} + b_{2i} ) (a_{2i-1} \eqnvar_{2i-1} \eqnvar_{2i} + b_{2i-1} \eqnvar_{2i-1} + c_{2i-1} \eqnvar_{2i} + d_{2i-1}) \\
=\ &(\eqnvar_{2i-1} a_{2i-1} + c_{2i-1} ) (\eqnvar_{2i}  (a_{2i} \eqnvar_{2i+1} + b_{2i}) + c_{2i} \eqnvar_{2i+1} + d_{2i}) \\
&-(\eqnvar_{2i+1} a_{2i} + b_{2i} ) (\eqnvar_{2i} (a_{2i-1} \eqnvar_{2i-1} + c_{2i-1} )+ b_{2i-1} \eqnvar_{2i-1} + d_{2i-1}) \\
=\ &(\eqnvar_{2i-1} a_{2i-1} + c_{2i-1} ) (c_{2i} \eqnvar_{2i+1} + d_{2i})\\
&- (\eqnvar_{2i+1} a_{2i} + b_{2i} ) (b_{2i-1} \eqnvar_{2i-1} + d_{2i-1}) \\
=\ & \eqnvar_{2i-1} \eqnvar_{2i+1}  ( a_{2i-1} c_{2i} - a_{2i} b_{2i-1} ) \\
&  + \eqnvar_{2i-1} ( a_{2i-1} d_{2i} - b_{2i-1} b_{2i})\\ 
&  + \eqnvar_{2i+1} ( c_{2i-1} c_{2i} - a_{2i} d_{2i-1})\\
&  + c_{2i-1} d_{2i} - d_{2i-1} b_{2i}.
\end{align*}

This eliminates every equation involving $\eqnvar_{2i}$. If $k$ is even, it results in a new equation system of size $k/2$. If $k$ is odd, we can do the same, but include the last equation in the new equation system.
Any solution of the old system is a solution of the new system.

Once two equations, as if $k=2$, remain, we eliminate the second variable by:
\begin{align*}
0 =\  &(a_1 \eqnvar_1+c_1) (a_2 \eqnvar_2 \eqnvar_1+b_2 \eqnvar_2+c_2 \eqnvar_1+d_2)\\
& -(a_2 \eqnvar_1+b_2)  (a_1 \eqnvar_1 \eqnvar_2 + b_1 \eqnvar_1 + c_1 \eqnvar_2 + d_1) \\
=\ &(a_1 \eqnvar_1+c_1) (\eqnvar_2 (a_2 \eqnvar_1+b_2)+c_2 \eqnvar_1+d_2)\\
&-(a_2 \eqnvar_1+b_2)  (\eqnvar_2(a_1 \eqnvar_1 +c_1)+ b_1 \eqnvar_1 + d_1)\\
=\ &(a_1 \eqnvar_1+c_1) (c_2 \eqnvar_1+d_2) - (a_2 \eqnvar_1+b_2)  (b_1 \eqnvar_1 + d_1)\\
=\ &\eqnvar^2_1 (a_1 c_2 - a_2 b_1) + \eqnvar_1 (a_1 d_2 + c_1 c_2 - a_2 d_1 - b_2 b_1) + c_1 d_2 - b_2 d_1.
\end{align*}
Finally, only one quadratic equation remains.

The coefficients of the resulting equations are exactly the determinants calculated in Definition~\ref{lem:tree:solution}. 
%
%
%
\end{proof}

The proof shows that any solution to the initial equation system of one missing cycle is a solution to the final quadratic equation of the recursion. Since a quadratic equation has at most two solutions, this means if the equation system has at least two solutions, the solutions of the quadratic equation are exactly the solutions of the equation system.

Unfortunately, if the equation system has exactly one solution, the recursion can introduce a spurious second solution. E.g.\ if it happens that $a_1 = a_2$ and $c_1 = b_2$ in the two equations case. Then $x_1 = - c_1 / a_1 =  - b_2/a_2$ is a solution to the final quadratic equation, regardless if it was a solution of the initial equation system.

This unfortunate case actually happened for some edges in Figure~\ref{fig:special:4:cycles} during our experiments. However, it does not affect the outcome of the algorithm since there always was an edge in the same cycle for which it did not happen. This edge provides a unique solution for all other edges in the cycle using propagate.

An alternate way of solving the missing cycle equation system is to insert the equation of propagate into the next equation. E.g.\ from 
$a_1 \eqnvar_1 \eqnvar_2 + b_1 \eqnvar_1 + c_1 \eqnvar_2 + d_1 = 0$  obtain 
$\eqnvar_2 = - (d_1 + b_1 \eqnvar_1) / (a_1 \eqnvar_1  + c_1 )$ and insert it into $a_2 \eqnvar_2 \eqnvar_1+b_2 \eqnvar_2+c_2 \eqnvar_1+d_2$. This returns $a_2 (- (d_1 + b_1 \eqnvar_1) / (a_1 \eqnvar_1  + c_1 )) \eqnvar_1+b_2 (- (d_1 + b_1 \eqnvar_1) / (a_1 \eqnvar_1  + c_1 ))+c_2 \eqnvar_1+d_2 = 0$, which can be solved for $\eqnvar_1$. But we do not recommend that approach, since it requires linearly many insertion steps in general unlike the recursion we  propose which only has logarithmic deep, so it has worse complexity and yields much larger expressions. 

%

\subsection{Proof of Proposition~\ref{prop:TreeID:ACID}}
\begin{proof}
The theorem follows from Proposition~\ref{prop:ACID:aIV}
which says that every edge in a tree graph identified by the ACID 
algorithm is  identified already in the  preliminary identification phase 
of the TreeID algorithm  (Lines 24-27).
We note that the phase of the algorithm runs in polynomial time.
\end{proof}

\subsection{Proof of Lemma~\ref{lem:path:transformation}}

\def\numberednode#1{#1}

\begin{proof}
The first part follows directly because nodes connected to all other nodes by bidirected edges do not contribute an equation to the equation system.

The second part: The equation of a missing bidirected edge $\numberednode {i+1}\bidirected \numberednode {j+1}$ 
contains factors $\sigma_{i+x,j+y}$ with $x,y\in\{0,1\}$ in $\cG$ 
and $\sigma_{\pi(i)+x,\pi(j)+y}$ in $\cG'$.
Then $i,j\in\bm$ and $\pi(i) + x = \pi(i + x)$ and $\pi(j) + y = \pi(j + y)$.

So the new equation system is obtained by replacing $\sigma_{i,j}$ with $\sigma_{\pi(i),\pi(j)}$, $\lambda_{i}$ with $\lambda_{\pi(i)}$, which is only a renaming of variables. Thus both equation systems have an isomorph solution space.
\end{proof}

\subsection{Proof of Proposition 5}

\begin{proof}
See Subsection~\ref{supplement:sec:experiments:canonical:path:graphs} below.
\end{proof}

\newpage

\section{EXPERIMENTS}

To test our algorithm TreeId on graphs, we have manually searched the missing cycles, and solved the resulting missing cycle equations. If multiple solutions occurred, we have inserted the solution of one missing cycle in the equations of other missing cycles. 

For symbolic calculations with polynomials, we have used the CAS (wx)Maxima. Rather than using PIT on FASTPs, we have fully expanded the equations in the CAS. Although this is slower, it provides more detailed information about the terms remaining in non-zero polynomials. A problem that occurs during the calculations is that Maxima converts $\sqrt{x^2}$ to $|x|$, so that $\sqrt{x^2} - x$ is non-zero. For TreeId to work, such equations should be considered as zero and need to be manually checked.

We only need to calculate one edge, since the identifiability of other edges should follow using propagation.

\subsection{Identification of $\cG_2$ in Figure~\ref{fig:examples}}\label{example:fig1:g2}

The graph $\cG_2$ in Figure~\ref{fig:examples} has a missing cycle $1 \bidirected 2 \bidirected 3 \bidirected 4 \bidirected 1$.

The recursion returns a quadratic equation $a \lambda_{01}^2 + b \lambda_{01} + c = 0 $ with

\catcode42=9

$a = ((σ11 * (- σ13) - σ11 * (- σ13)) * ((- σ44) * (- σ15) - σ14 * σ45) - (σ14 * (- σ15) - σ14 * (- σ15)) * (σ11 * σ34 - (- σ14) * (- σ13)))$

$b = ((σ11 * (- σ13) - σ11 * (- σ13)) * ((- σ44) * σ25 - (- σ24) * σ45) - (σ14 * σ25 - (- σ24) * (- σ15)) * (σ11 * σ34 - (- σ14) * (- σ13))) + (((- σ12) * (- σ13) - σ11 * σ23) * ((- σ44) * (- σ15) - σ14 * σ45) - (σ14 * (- σ15) - σ14 * (- σ15)) * ((- σ12) * σ34 - (- σ14) * σ23))$

$c = (((- σ12) * (- σ13) - σ11 * σ23) * ((- σ44) * σ25 - (- σ24) * σ45) - (σ14 * σ25 - (- σ24) * (- σ15)) * ((- σ12) * σ34 - (- σ14) * σ23))$

Simplifying this equation  in (wx)Maxima returns:

$a = 0$

$b = 	(σ12*σ13-σ11*σ23)*(σ15*σ44-σ14*σ45)-(σ14*σ25-σ15*σ24)*(σ11*σ34-σ13*σ14)$

$c = (σ12*σ13-σ11*σ23)*(σ24*σ45-σ25*σ44)-(σ14*σ25-σ15*σ24)*(σ14*σ23-σ12*σ34)$

Thus $\lambda_{01} = - c / b$ is a unique solution. The fraction is valid  because $\evalSigma{b} = \omega_{33}*\omega_{01}*\omega_{02}*\omega_{04} $ is not zero.

\subsection{Identification of the graphs in Figure~\ref{fig:examples:drton}}\label{example:fig3}

Here we investigate the 5 tree graphs of \citep{identifyingEdgeWiseDeterminantalDrton}.

\expandafter\def\csname pos0x\endcsname{-0.00cm}\expandafter\def\csname pos0y\endcsname{2.00cm}
\expandafter\def\csname curve1x\endcsname{-1.23cm}\expandafter\def\csname curve1y\endcsname{1.70cm}
\expandafter\def\csname pos1x\endcsname{-1.90cm}\expandafter\def\csname pos1y\endcsname{0.62cm}
\expandafter\def\csname curve3x\endcsname{-2.00cm}\expandafter\def\csname curve3y\endcsname{-0.65cm}
\expandafter\def\csname pos3x\endcsname{-1.18cm}\expandafter\def\csname pos3y\endcsname{-1.62cm}
\expandafter\def\csname pos4x\endcsname{1.18cm}\expandafter\def\csname pos4y\endcsname{-1.62cm}
\expandafter\def\csname curve4x\endcsname{2.00cm}\expandafter\def\csname curve4y\endcsname{-0.65cm}
\expandafter\def\csname pos2x\endcsname{1.90cm}\expandafter\def\csname pos2y\endcsname{0.62cm}
\expandafter\def\csname curve2x\endcsname{1.23cm}\expandafter\def\csname curve2y\endcsname{1.70cm}
\def\thenodes{\circlednode{0}\circlednode{1}\circlednode{2}\circlednode{3}\circlednode{4}}

\subsubsection{Identification of (4680, 403)}
\placeFiveNodes%
\begin{center}
\makegraph{\edge{0}{1}\edge{0}{2}\edge{0}{3}\edge{0}{4}\edge{1}{3}\markI{4}\markZ{1}\markZ{2}\markZ{3}}
\end{center}

There are three missing cycles, $1\bidirected 2\bidirected 4$, $2\bidirected 3\bidirected 4$, and $1\bidirected 2\bidirected 3\bidirected 4$. Each cycle yields two solutions. 
The solutions for $\lambda_1$ of $1\bidirected 2 \bidirected 4$ are

\def\sqrtvar{s}

%

$\lambda_1=(\sqrt{\sqrtvar}+(σ12*σ24+σ14*σ22)*σ35+(-σ12*σ25-σ15*σ22)*σ34-σ14*σ23*σ25+σ15*σ23*σ24)/(2*σ12*σ14*σ35-2*σ12*σ15*σ34-2*σ13*σ14*σ25+2*σ13*σ15*σ24)$, and

$\lambda'_1=(-\sqrt{\sqrtvar}+(σ12*σ24+σ14*σ22)*σ35+(-σ12*σ25-σ15*σ22)*σ34-σ14*σ23*σ25+σ15*σ23*σ24)/(2*σ12*σ14*σ35-2*σ12*σ15*σ34-2*σ13*σ14*σ25+2*σ13*σ15*σ24)$

where $\sqrtvar = (-σ14*(σ23*σ25-σ22*σ35)+σ24*(σ12*σ35-σ13*σ25)-σ15*(σ22*σ34-σ23*σ24)+σ25*(σ13*σ24-σ12*σ34))^2-4*(-σ14*(σ12*σ35-σ13*σ25)-σ15*(σ13*σ24-σ12*σ34))*(σ24*(σ23*σ25-σ22*σ35)+σ25*(σ22*σ34-σ23*σ24))$.

The solutions are distinct because the expression $\evalSigma{\sqrt{s}}$ simplifies to the non-zero polynomial
$\sqrt{{{{\omega_{04}}}^{2}} ( {\lambda_{01}} {\omega_{02}} {\omega_{13}}+{\lambda_{01}} {{{\lambda_{12}}}^{2}} {\lambda_{23}} {\omega_{01}}-{\lambda_{01}} {\lambda_{23}} {\omega_{01}}+{{{\lambda_{01}}}^{2}} {{{\lambda_{12}}}^{2}} {\lambda_{23}}-{{{\lambda_{12}}}^{2}} {\lambda_{23}}-{{{\lambda_{01}}}^{2}} {\lambda_{23}}+{\lambda_{23}})^{2}}$.




The former $\lambda_1$ is also a solution for the cycle  $1\bidirected 2\bidirected 3\bidirected 4$. The latter $\lambda'_1$ is not. Thus $\lambda_1$ is the true solution.

\subsubsection{Identification of (4680, 914)}
%
%
%
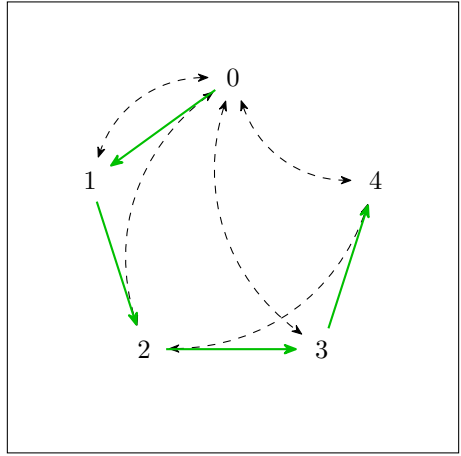
\begin{center}
\makegraph{\edge{0}{1}\edge{0}{2}\edge{0}{3}\edge{0}{4}\edge{2}{4}\markI{4}\markZ{1}\markZ{2}\markZ{3}}
\end{center}

There are three missing cycles $1\bidirected 2\bidirected 3$, $2\bidirected 3 \bidirected 4$, and $1 \bidirected 2 \bidirected 3 \bidirected 4$.
The missing cycle $1\bidirected 2\bidirected 3$ gives two solutions:

$\lambda_1=(\sqrt\sqrtvar+(σ12*σ23+σ13*σ22)*σ34+(-σ12*σ24-σ14*σ22)*σ33-σ13*σ23*σ24+σ14*σ23^2)/(2*σ12*σ13*σ34-2*σ12*σ14*σ33-2*σ13^2*σ24+2*σ13*σ14*σ23)$, and

$\lambda'_1=(-\sqrt\sqrtvar+(σ12*σ23+σ13*σ22)*σ34+(-σ12*σ24-σ14*σ22)*σ33-σ13*σ23*σ24+σ14*σ23^2)/(2*σ12*σ13*σ34-2*σ12*σ14*σ33-2*σ13^2*σ24+2*σ13*σ14*σ23)$

where $\sqrtvar = (σ12^2*σ23^2-2*σ12*σ13*σ22*σ23+σ13^2*σ22^2)*σ34^2+(((2*σ12*σ13*σ22-2*σ12^2*σ23)*σ24+2*σ12*σ14*σ22*σ23-2*σ13*σ14*σ22^2)*σ33+(2*σ13^2*σ22*σ23-2*σ12*σ13*σ23^2)*σ24+2*σ12*σ14*σ23^3-2*σ13*σ14*σ22*σ23^2)*σ34+(σ12^2*σ24^2-2*σ12*σ14*σ22*σ24+σ14^2*σ22^2)*σ33^2+((2*σ12*σ13*σ23-4*σ13^2*σ22)*σ24^2+(6*σ13*σ14*σ22*σ23-2*σ12*σ14*σ23^2)*σ24-2*σ14^2*σ22*σ23^2)*σ33+σ13^2*σ23^2*σ24^2-2*σ13*σ14*σ23^3*σ24+σ14^2*σ23^4$.

%
%
%

The solutions are distinct because $\evalSigma{\sqrt\sqrtvar}$ simplifies to non-zero $\sqrt{{{\left( {\lambda_{01}} {\omega_{01}}+\omega_{11}\right) }^{2}} {{\left( 2 {\lambda_{01}} {\lambda_{12}} {\omega_{02}}+\omega_{22}\right) }^{2}} {{{\omega_{03}}}^{2}}}$. 

If we insert this in the cycle $1\bidirected 2\bidirected3 \bidirected 4$, maxima says neither is a solution because there are terms involving   $\mid w_{03}\mid$ that do not cancel. Manually replacing them with  $(w_{03})$ discovers that only the first $\lambda_1$ is a solution. Thus the graph is fully identifiable.

\subsubsection{Identification of (360, 117)}

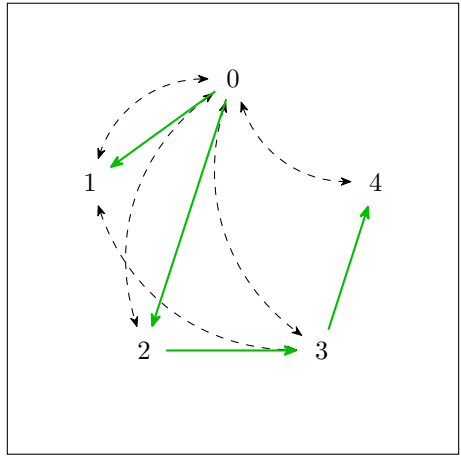
\begin{center}
\makegraph{\edge{0}{1}\edge{0}{2}\edge{0}{3}\edge{0}{4}\edge{1}{3}\markZ{1}\markctreeedge[green!75!black]{0}{2}\markZ{3}\markI{4}}
\end{center}

The missing cycle $1\bidirected 2\bidirected 4$ gives two solutions:

$\lambda_1=(\sqrt{\sqrtvar}+(σ11*σ24+σ12*σ14)*σ35+(-σ11*σ25-σ12*σ15)*σ34+σ13*σ14*σ25-σ13*σ15*σ24)/(2*σ11*σ14*σ35-2*σ11*σ15*σ34)$, and

$\lambda'_1=(-\sqrt{\sqrtvar}+(σ11*σ24+σ12*σ14)*σ35+(-σ11*σ25-σ12*σ15)*σ34+σ13*σ14*σ25-σ13*σ15*σ24)/(2*σ11*σ14*σ35-2*σ11*σ15*σ34)$

where $\sqrtvar = ((σ11^2*σ24^2-2*σ11*σ12*σ14*σ24+σ12^2*σ14^2)*σ35^2+(((2*σ11*σ12*σ14-2*σ11^2*σ24)*σ25+2*σ11*σ12*σ15*σ24-2*σ12^2*σ14*σ15)*σ34+(2*σ11*σ13*σ14*σ24-4*σ11*σ14^2*σ23+2*σ12*σ13*σ14^2)*σ25-2*σ11*σ13*σ15*σ24^2+(4*σ11*σ14*σ15*σ23-2*σ12*σ13*σ14*σ15)*σ24)*σ35+(σ11^2*σ25^2-2*σ11*σ12*σ15*σ25+σ12^2*σ15^2)*σ34^2+(-2*σ11*σ13*σ14*σ25^2+(2*σ11*σ13*σ15*σ24+4*σ11*σ14*σ15*σ23-2*σ12*σ13*σ14*σ15)*σ25+(2*σ12*σ13*σ15^2-4*σ11*σ15^2*σ23)*σ24)*σ34+σ13^2*σ14^2*σ25^2-2*σ13^2*σ14*σ15*σ24*σ25+σ13^2*σ15^2*σ24^2)$




The solutions are distinct because $\evalSigma{\sqrt{\sqrtvar}}$ simplifies to non-zero $\sqrt{{{{\omega_{15}}}^{2}} {{\left( {\omega_{13}} {\omega_{24}}+2 {\lambda_{13}} {\lambda_{34}} {\omega_{12}} {\omega_{13}}+{\omega_{33}}\, {\lambda_{34}} {\omega_{12}}\right) }^{2}}}$.

Only the first solution is valid for the 4-cycle $1\bidirected 2\bidirected 3 \bidirected 4$, so the graph is fully identifiable.

\subsubsection{Identification of (360, 369)}

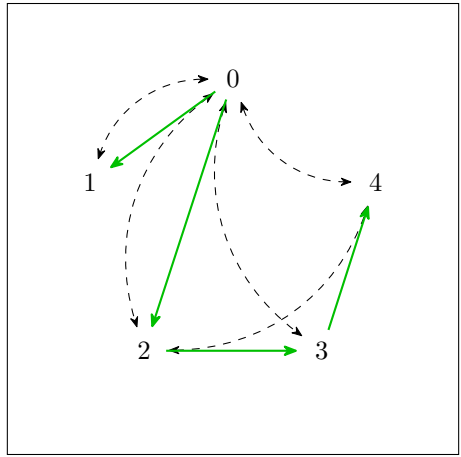
\begin{center}
\makegraph{\edge{0}{1}\edge{0}{2}\edge{0}{3}\edge{0}{4}\edge{2}{4}\markZ{1}\markctreeedge[green!75!black]{0}{2}\markZ{3}\markI{4}}
\end{center}
%
%

There are three missing cycles $1\bidirected 2\bidirected 3$, $1\bidirected 3 \bidirected 4$, and $1 \bidirected 2 \bidirected 3 \bidirected 4$.
The missing cycle $1\bidirected 2\bidirected 3$ gives two solutions:

$\lambda_1=(\sqrt{\sqrtvar}+(σ11*σ23+σ12*σ13)*σ34+(-σ11*σ24-σ12*σ14)*σ33+σ13^2*σ24-σ13*σ14*σ23)/(2*σ11*σ13*σ34-2*σ11*σ14*σ33)]$, and

$\lambda'_1=(-\sqrt{\sqrtvar}+(σ11*σ23+σ12*σ13)*σ34+(-σ11*σ24-σ12*σ14)*σ33+σ13^2*σ24-σ13*σ14*σ23)/(2*σ11*σ13*σ34-2*σ11*σ14*σ33)$

where $\sqrtvar=(σ11^2*σ23^2-2*σ11*σ12*σ13*σ23+σ12^2*σ13^2)*σ34^2+(((2*σ11*σ12*σ13-2*σ11^2*σ23)*σ24+2*σ11*σ12*σ14*σ23-2*σ12^2*σ13*σ14)*σ33+(2*σ12*σ13^3-2*σ11*σ13^2*σ23)*σ24+2*σ11*σ13*σ14*σ23^2-2*σ12*σ13^2*σ14*σ23)*σ34+(σ11^2*σ24^2-2*σ11*σ12*σ14*σ24+σ12^2*σ14^2)*σ33^2+(-2*σ11*σ13^2*σ24^2+(6*σ11*σ13*σ14*σ23-2*σ12*σ13^2*σ14)*σ24-4*σ11*σ14^2*σ23^2+2*σ12*σ13*σ14^2*σ23)*σ33+σ13^4*σ24^2-2*σ13^3*σ14*σ23*σ24+σ13^2*σ14^2*σ23^2$.


%

The solutions are distinct because $\evalSigma{\sqrt{\sqrtvar}}$ simplifies to non-zero $\sqrt{ {{{\omega_{01}}}^{2}} {{\left( 2 {\lambda_{02}} {\omega_{02}}+{\omega_{22}}\right) }^{2}} {{{\omega_{33}}}^{2}} }$.

If we insert this in the cycle $1\bidirected 2 \bidirected 3 \bidirected 4$, maxima says neither is a solution because there are terms involving $\mid \omega_{03}\mid$ that do not cancel. Manually replacing $\mid \omega_{03}\mid$ by $(\omega_{03})$ discovers that only the first $\lambda_1$ is a solution. Thus the graph is fully identifiable. (alternative the $\lambda_1$ from cycle 134 works without manual replacement)

\subsubsection{Identification of (840, 466)}

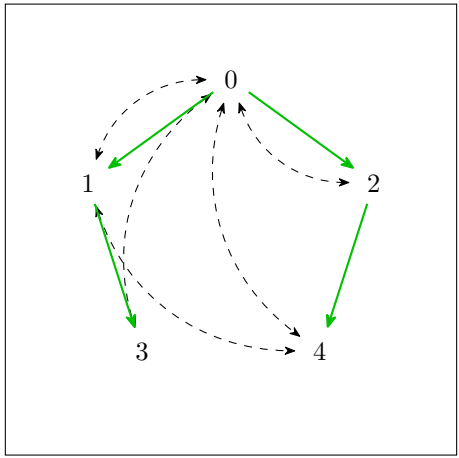
\placeFiveTreeNodes

\begin{center}
\makegraph{\edge{0}{1}\edge{0}{2}\edge{0}{3}\edge{0}{4}\edge{1}{4}%
\markI{1}\markctreeedge[green!75!black]{1}{3}\markctreeedge[green!75!black]{0}{2}\markctreeedge[green!75!black]{2}{4}%
}\end{center}

There are three missing cycles $1\bidirected 2\bidirected 3$, $2\bidirected 3 \bidirected 4$, and $2 \bidirected 1 \bidirected 3 \bidirected 4$.
The missing cycle $1\bidirected 2\bidirected 3$ gives two solutions:

$\lambda_1=(\sqrt{\sqrtvar}+(σ11*σ22+σ12^2)*σ34+(σ12*σ13-σ11*σ23)*σ24-σ12*σ14*σ23-σ13*σ14*σ22)/(2*σ11*σ12*σ34-2*σ11*σ14*σ23)$, and

$\lambda'_1=-(\sqrt{\sqrtvar}+(-σ11*σ22-σ12^2)*σ34+(σ11*σ23-σ12*σ13)*σ24+σ12*σ14*σ23+σ13*σ14*σ22)/(2*σ11*σ12*σ34-2*σ11*σ14*σ23)$

where $\sqrtvar=(σ11^2*σ22^2-2*σ11*σ12^2*σ22+σ12^4)*σ34^2+(((-2*σ11^2*σ22-2*σ11*σ12^2)*σ23+2*σ11*σ12*σ13*σ22+2*σ12^3*σ13)*σ24+(6*σ11*σ12*σ14*σ22-2*σ12^3*σ14)*σ23-2*σ11*σ13*σ14*σ22^2-2*σ12^2*σ13*σ14*σ22)*σ34+(σ11^2*σ23^2-2*σ11*σ12*σ13*σ23+σ12^2*σ13^2)*σ24^2+(2*σ11*σ12*σ14*σ23^2+(2*σ11*σ13*σ14*σ22-2*σ12^2*σ13*σ14)*σ23-2*σ12*σ13^2*σ14*σ22)*σ24+(σ12^2*σ14^2-4*σ11*σ14^2*σ22)*σ23^2+2*σ12*σ13*σ14^2*σ22*σ23+σ13^2*σ14^2*σ22^2$.

%

The solutions are distinct because $\evalSigma{\sqrt\sqrtvar}$ simplifies to non-zero $\sqrt{{{\left( 2 {\lambda_{01}} {\omega_{01}}+\omega_{11}\right) }^{2}} {{{\omega_{02}}}^{2}} {{{\omega_{03}}}^{2}}}$.

If we insert this in the cycle $1\bidirected 3 \bidirected 4\bidirected 2$, maxima says neither is a solution because there are terms involving $\mid \omega_{02}\mid$ and  $\mid \omega_{03}\mid$ that do not cancel. Manually replacing them with $(\omega_{02})$ and $(\omega_{03})$ discovers that only the first $\lambda_1$ is a solution. Thus the graph is fully identifiable.

\subsection{Canonical path graphs}\label{supplement:sec:experiments:canonical:path:graphs}

For Proposition 5, we have calculated the identifiability of all canonical path graphs with exactly one missing cycle of length at most 10 using Gröbner bases. 

 The results are shown in the file \texttt{canonical-cycles.pdf}.



The directed edges form a path $0 \to 1 \to 2 \to 3 \to \ldots$ in each graph.

The color of the directed edges encodes the identifiability:

\begin{enumerate}
\item green: The edge is uniquely identifiable.
\item yellow: The edge is 2-identifiable.
\item blue: The Gröbner base does not include a solution the edge identifiability directly, but the edge can be identified using some kind of propagate from adjacent edges. It is either 1-identifiable or 2-identifiable depending on the color of the adjacent edge.
\item black: The edge is not identifiable
\end{enumerate}

The number of nodes is not specified, since additional nodes do not affect the identifiability

The graphs are normalized for shifting, but not for permutations. E.g.\ the 2nd and 3rd graph of length 3 in \texttt{canonical-cycles.pdf} are equivalent under permutations, and so must have the same identifiability. Hence here one can see more fully identifiable graphs of length 4 than in the main paper, which is not a contradiction, as they are equivalent.

\subsection{Comparison with Gröbner bases}\label{sec:dagitty:vs:gröbner}

Besides path graphs with a single missing cycle, we have enumerated 879 path graphs with 8 nodes and various combinations of bidirected edges. On these graphs, we have searched the identifiable edges with TreeID and a Gröbner base approach. The results are shown in the pdf files of the folder \texttt{879graphs}.

TreeID (as implemented in DAGitty) completed its computation in a day and night on a laptop. To calculate the Gröbner bases, we have used Singular \citep{singular2016}. Over several months on a desktop PC, it calculated the Gröbner bases for the first 587 graphs and then it stopped proceeding, possibly having exhausted the available RAM. So we have aborted it, and continued the computations for the remaining graphs with a time limit of 4 hours per graph on a server.  There it eventually finished.

Comparing the output of both algorithms, we see that TreeID can identify all edges that can be identified using Gröbner bases on these graphs. 

Furthermore, it can identify edges that could not be identified with the Gröbner bases, which should be impossible. This seems to occur due to two bugs in our Gröbner base analysis: 1) when the 4 hour time limit was breached, we recorded all edges as unidentifiable rather than a computation failure, and  2) prioritizing base polynomials containing a single variable over propagation. For example, the Gröbner base might contain $p = \lambda_2^2 \Sigma_1 + \lambda_2 \Sigma_2 + \Sigma_3$, $q = \lambda_2 \Sigma_4 + \lambda_1 \Sigma_5 + \Sigma_6$ and $r = \lambda_1 \Sigma_7 + \Sigma_8$, where $\Sigma_i$ are arbitrary polynomials in $\sigma_{j,k}$. From $p$, our analysis script would conclude that $\lambda_2$ is 2-identifiable, and from $r$ that $\lambda_1$ is 1-identifiable. Our script would then stop, assuming it had discovered the identifiability of all edges. However, from $q$ and $r$ together, one can conclude that $\lambda_2$ is also 1-identifiable. As explained in the previous section, in our visualizations we have drawn edges identified by $p$, $q$, resp. $r$ as yellow, blue, resp. green. Essentially the bug was to draw edges that might be blue or yellow as yellow, even if blue was better.

\subsection{wxMaxima files}
\catcode42=12

We have performed several calculations in the  wxMaxima CAS. For the sake of reproducibility, we share the following wxMaxima files: 

\begin{enumerate}
\item \texttt{3nodes.wxmx}, \texttt{5nodes.wxmx}, \texttt{binaryTree.wxmx}, \texttt{binaryTree2.wxmx}: Examples of the calculation of the $\Sigma$ matrix.
\item \texttt{propagate.wxmx}: Shows that the propagation step of a FASTP results in a FASTP.
\item \texttt{recursion.wxmx} and \texttt{recursion-abcd.wxmx}: A general recursion scheme for a cycle of 3,4, and 6 length. In the first level, the $a,b,c,d$ have not been replaced by any $\sigma_{ij}$, so the recursion returns a general quadratic equation. In this general quadratic equation, the $a,b,c,d$ can be replaced by $\sigma_{ij}$ to obtain a quadratic equation for a specific system without performing the recursion.
\item \texttt{substitution-scheme-3-equations.wxmxm}, \texttt{substitution-scheme-4-equations.wxmx}: An alternative way of solving the equation system of a missing cycle using substitution rather than the recursion.
\item \texttt{substitution-abcd.wxmxm}: Shows that substitution unlike the recursion does not work in general.
\item \texttt{test-drton-tsiv-*.wxmx}: The calculations for Figure~\ref{fig:examples:drton}. 
\item \texttt{test-drton-tsiv-Fig2*.wxmx}: The calculations for $\cG_2$ in Figure~\ref{fig:examples} (which is a graph in Figure~2 in \cite{identifyingEdgeWiseDeterminantalDrton}).
\item \texttt{test-discriminant-cycle*.wxmx}: The calculations for Figure~4 (canonical path graphs with no solutions or a unique solution).
\end{enumerate}

\subsection{Script files}

In the folder \texttt{scripts}, you can find some of the code we have used to run the experiments.

\begin{enumerate}
\item \texttt{singlecyclepathgraphs.lpr}: Creates path graphs with a single missing cycle.

Example usage after compilation: \texttt{./singlecyclepathgraphs 4}

\item \texttt{identification-helper.xqm}: Various functions used by the other scripts, mostly to convert graphs from one file format to another. The functions can also be called directly:

Example: \texttt{xidel --module identification-helper.xqm -e 'helper:drton-model-to-pretty-edge-\\string(5, "(4680, 403)")'} to convert the graph $(4680, 403)$ in the notation of \cite{identifyingEdgeWiseDeterminantalDrton} to our format.

\item \texttt{drtonModelToGraph.lpr}: Another program to convert the graphs of \cite{identifyingEdgeWiseDeterminantalDrton} to our format. The input graph is specified as constants in the source code.

\item \texttt{identifiability-singular-model.xq}: Creates Singular commands to calculate the Gröbner base for some graphs. It outputs variables \texttt{L}$i$ for $\lambda_i$  and \texttt{s}$i$\texttt{s}$j$ for $\sigma_{ij}$.

Example: \texttt{echo '[[[1, 2], [1, 4], [2,3], [2,4], [3,4]]]' | xidel --input-format json  - -e @identifiability-singular-model.xq | Singular}

\item \texttt{identifiable-iff.compress.pl}: Removes all sigma variables from the Gröbner base output of Singular to save space.
\item \texttt{identifiable-iff.parsesingular.pl}: Parses the Gröbner base output to find the identifiable edges (see \ref{supplement:sec:experiments:canonical:path:graphs} and \ref{sec:dagitty:vs:gröbner}). 

Example: \texttt{ Singular < input-with-the-singular-model | perl identifiable-iff.compress.pl | perl identifiable-iff.parsesingular.pl > output.tex}

The output consists of TeX commands which can create the visualizations in the folder \texttt{879graphs} when combined with suitable definitions for the commands.

\item \texttt{makegraph-results-to-json.xq}: This converts the output of \texttt{identifiable-iff.parsesingular.pl} to JSON. The JSON can be copied into a JavaScript program, from which it is easy to call DAGitty.

\item \texttt{identifiable-iffgraphs-cycles-solution.xq}: Calculates the quadratic equation of Theorem~\ref{lem:tree:solution} for a given graph and missing cycle using the recursion of Definition~\ref{def:polynomials:abcd}.

Example: \texttt{model='1->2 1->3 1->4 4->5 1<->2 1<->3 1<->4 1<->5' cycle='2 3 4'  xidel identifiable-iffgraphs-cycles-solution.xq}

The last three lines of the output can be copied verbatim into Maxima to define the variables $a,b,c$ for the quadratic equation. This is the script we have used to calculate the solutions in \ref{example:fig1:g2} and \ref{example:fig3}, in combination with the next two scripts which reveal how many solutions the quadratic equation has.

\item \texttt{graph-to-matrices.xq}: Creates Maxima commands to calculate $\Lambda,\Omega,\Sigma$ matrices for a given graph.

Example: \texttt{model='[[1, 2], [2, 3], [1, 3]]' xidel graph-to-matrices.xq}

\item \texttt{discriminant.xq}: Creates Maxima commands to substitute the elements of a $\Sigma$ matrix of size \texttt{n} into the expressions $a$ and $b^2 - 4ac$ of Lemma~\ref{lemm:number:of solutions}.

Example: \texttt{n=5 xidel discriminant.xq}


\end{enumerate}

\texttt{*.pl} files are run with Perl, \texttt{*.xq} files are run with Xidel, \texttt{*.lpr} files  are compiled with FreePascal.

There are two different formats used to read to graphs in the scripts. A JSON syntax that only lists the missing bidirected edges of a graph (e.g. \texttt{[[1,2]]} for a missing edge $1\bidirected 2$) and a format that lists all existing edges (e.g. \texttt{1->2 2<->3}). Some scripts assume the root node is node $1$ (especially those creating commands for Maxima and Singular), some scripts assume the root node is node $0$. The user needs to pay attention to this when using the scripts. 

\end{document}